%% file: spca-scca-v9.tex
\documentclass[11pt]{article}

\usepackage{url,fancybox}

\usepackage{lscape}

\usepackage{booktabs}       
\usepackage{amsfonts}       
\usepackage{nicefrac}       
\usepackage{microtype}
\usepackage{graphicx} 
\usepackage{subfigure}
\usepackage{graphics}

\usepackage{algorithm}
\usepackage{algorithmic}

\usepackage{amsmath}
\usepackage{mathrsfs}
\usepackage{amssymb}
\usepackage{amsthm}
\usepackage[square,sort,comma,numbers]{natbib}

\usepackage{color}
\usepackage{cleveref}

\usepackage[left=1in,top=1in,right=1in,bottom=1in,letterpaper]{geometry}
\usepackage{subfigure}
\usepackage{footnote}
\usepackage[normalem]{ulem}

\usepackage{multirow}
\usepackage{bigstrut}
\usepackage{longtable}
\usepackage{pdfpages}

\numberwithin{equation}{section}

\newtheorem{theorem}{Theorem}[section]

\newtheorem{lemma}[theorem]{Lemma}

\newtheorem{definition}[theorem]{Definition}

\newtheorem{assumption}[theorem]{Assumption}
\newtheorem{remark}[theorem]{Remark}

\newcommand{\M}{\mathcal{M}}
\newcommand{\G}{\mathcal{G}}
\newcommand{\J}{\mathcal{J}}
\newcommand{\GG}{\mathrm{G}}
\newcommand{\T}{\mathrm{T}}
\newcommand{\R}{\mathbb{R}}

\newcommand{\etal}{et al. }
\newcommand{\argmin}{\mathop{\rm argmin}}
\newcommand{\argmax}{\mathop{\rm argmax}}
\newcommand{\br}{\mathbb{R}}

\newcommand{\half}{\frac{1}{2}}

\newcommand{\Tr}{\mathrm{Tr}}
\newcommand{\st}{\mathrm{s.t. }}
\newcommand{\ie}{\mathrm{i.e. }}
\newcommand{\vet}{\mathrm{vec}}

\newcommand{\conv}{\mathrm{conv}}

\newcommand{\be}{\begin{equation}}
\newcommand{\ee}{\end{equation}}
\newcommand{\ba}{\begin{array}}
\newcommand{\ea}{\end{array}}
\newcommand{\bad}{\begin{aligned}}
\newcommand{\ead}{\end{aligned}}
\newcommand{\normone}[1]{\| #1 \|_1}
\newcommand{\normtwo}[1]{\| #1 \|}
\newcommand{\abs}[1]{| #1 |}

\newcommand{\normfro}[1]{\| #1 \|_{F}}
\newcommand{\inp}[2]{\langle #1, #2 \rangle}
\newcommand{\setword}[2]{\phantomsection #1\def\@currentlabel{\unexpanded{#1}}\label{#2}}

\newcommand{\St}{\mathrm{St}}
\newcommand{\GSt}{\mathrm{GSt}}

\newcommand{\LCal}{\mathcal{L}}
\newcommand{\grad}{\mathrm{grad}}
\newcommand{\vvec}{\mathrm{vec}}
\newcommand{\svec}{\overline{\mathrm{vec}}}
\newcommand{\prox}{\mathrm{prox}}
\newcommand{\Proj}{\mathrm{Proj}}
\newcommand{\bx}{\mathbf{x}}

\begin{document}
\title{An Alternating Manifold Proximal Gradient Method \\ for Sparse PCA and Sparse CCA}
\author{Shixiang Chen\thanks{Department of Systems Engineering and Engineering Management, The Chinese University of Hong Kong}
\and Shiqian Ma\thanks{Department of Mathematics, University of California, Davis}
\and Lingzhou Xue\thanks{Department of Statistics, The Pennsylvania State University}
\and Hui Zou\thanks{School of Statistics, University of Minnesota}}
\date{March 27, 2019}
\maketitle

\begin{abstract}
Sparse principal component analysis (PCA) and sparse canonical correlation analysis (CCA) are two essential techniques from high-dimensional statistics and machine learning for analyzing large-scale data. Both problems can be formulated as an optimization problem with nonsmooth objective and nonconvex constraints. Since non-smoothness and nonconvexity bring numerical difficulties, most algorithms suggested in the literature either solve some relaxations or are heuristic and lack convergence guarantees. In this paper, we propose a new alternating manifold proximal gradient method to solve these two high-dimensional problems and provide a unified convergence analysis. Numerical experiment results are reported to demonstrate the advantages of our algorithm.
\end{abstract}

\section{Introduction}

Principal Component Analysis (PCA), invented by Pearson \citep{pearson1901liii}, is widely used in dimension reduction. Let $X = [X_1, \ldots, X_p] \in\br^{n\times p}$ be a given data matrix whose column means are all 0. Assume the singular value decomposition (SVD) of $X$ is $X = UDV^\top$, then it is known that $Z = UD$ are the principal components (PCs) and the columns of $V$ are the corresponding loadings of the PCs. In other words, the first PC can be defined as $Z_1 = \sum_{j=1}^p \alpha_{1j}X_j$ with $\alpha_1 = (\alpha_{11},\ldots,\alpha_{1p})^\top$ maximizing the variance of $Z_1$, i.e.,
\[\alpha_1 = \argmax_\alpha \alpha^\top\hat{\Sigma}\alpha, \quad \st, \|\alpha_1\|_2=1,\]
where $\hat{\Sigma}=(X^\top X)/(n-1)$ is the sample covariance matrix. The rest PCs are defined as
\[\alpha_{k+1} = \argmax_\alpha \alpha^\top\hat{\Sigma}\alpha, \quad \st, \|\alpha\|_2=1, \alpha^\top \alpha_l=0, \forall 1\leq l\leq k.\]

Canonical correlation analysis (CCA), introduced by Hotelling \cite{hotelling1936relations}, is another widely used tool, which explores the
relation between two sets of variables. For random variables $x\in\br^p$ and $y\in\br^q$, CCA seeks linear combinations of $x$ and $y$ such that the resulting values are mostly correlated. That is, it targets to solve the following optimization problem:
\be\label{cca-single}
\max_{u\in\br^p,v\in\br^q}\frac{u^\top\Sigma_{xy}v}{\sqrt{u^\top\Sigma_x u}\sqrt{v^\top\Sigma_y v}},
\ee
where $\Sigma_x$ and $\Sigma_y$ are covariance of $x$ and $y$ respectively, $\Sigma_{xy}$ is their covariance matrix, and $u\in\br^p$ and $v\in\br^q$ are the first canonical vectors.
It can be shown that solving \eqref{cca-single} corresponds to computing the SVD of $\Sigma_x^{-1/2}\Sigma_{xy}\Sigma_y^{-1/2}$. In practice, given two centered data sets $X\in\R^{n\times p},Y\in\R^{n\times q}$ with joint covariance matrix
\[
 \left(
 \begin{matrix}
     \Sigma_x & \Sigma_{xy}\\
     \Sigma_{yx} &\Sigma_y
 \end{matrix}\right),
\]
CCA seeks the coefficients $u$, $v$ such that the correlation of $Xu$ and $Yv$ is maximized. The classical CCA \cite{hotelling1936relations} can be formulated as
\be\label{cca}
\ba{ll}
  \max_{u\in\br^p,v\in\br^q} & u^\top X^\top Yv \\
  \st        & u^\top X^\top X u = 1, \ v^\top Y^\top Y v = 1,
\ea
\ee
where $X^\top Y, X^\top X, Y^\top Y$ are used to estimated the true parameters $\Sigma_{xy},\Sigma_x,\Sigma_y$ after scaling. 

However, PCA and CCA perform poorly and often lead to wrong findings when modeling with high-dimensional data. For example, when the dimension is proportional to the sample size such that $\lim_{n\to \infty} p/n=\gamma\in(0,1)$ and the largest eigenvalue $\lambda_1\le\sqrt{\gamma}$, the leading sample principal eigenvector could be asymptotically orthogonal to the leading population principal eigenvector almost surely \cite{baik2006eigenvalues,paul2007asymptotics,nadler2008finite}. Sparse PCA and Sparse CCA are proposed as the more interpretable and reliable dimension reduction and feature extraction techniques for high-dimensional data. In what follows, we provide a brief overview of their methodological developments respectively.

\textbf{Sparse PCA} seeks sparse basis (loadings) of the subspace spanned by the data so that the obtained leading PCs are easier to interpret.
Jolliffe \etal \cite{Jolliffe2003} proposed the SCoTLASS procedure by imposing $\ell_1$ norm on the loading vectors, which can be formulated as the following optimization problem for given data $X\in\br^{n\times p}$:
\be\label{scotlass}
\ba{ll}
\min_{A\in\br^{p\times r}} & -\Tr(A^\top X^\top X A) + \mu\|A\|_1 \\
\st                        & A^\top A = I_{r},
\ea
\ee
where $\Tr(Z)$ denotes the trace of matrix $Z$, $\mu>0$ is a weighting parameter, $\|A\|_1 = \sum_{ij}|A_{ij}|$, and $I_r$ denotes the $r\times r$ identity matrix. Note that the original SCoTLASS model in \cite{Jolliffe2003} uses an $\ell_1$ constraint $\|A\|_1\leq t$ instead of penalizing $\|A\|_1$ in the objective. The SCoTLASS model \eqref{scotlass} is numerically very challenging. Algorithms for solving it have been very limited. 
As a result, a new formulation of Sparse PCA has been proposed by Zou \etal \cite{Zou-spca-2006}, and it has been the main focus in the literature on this topic.
In \cite{Zou-spca-2006}, Zou \etal formulate Sparse PCA problem as the following ridge regression problem plus a lasso penalty:
\be\label{Zou-spca}
\ba{cl}
\min_{A\in\br^{p\times r},B\in\br^{p\times r}} & H(A,B) + \mu\sum_{j=1}^r\|B_j\|_2^2+\sum_{j=1}^r\mu_{1,j}\|B_j\|_1 \\
\st        & A^\top A = I_r,
\ea
\ee
where
\be\label{def:H} H(A,B) := \sum_{i=1}^n \|\bx_i-AB^\top \bx_i\|_2^2,\ee
$\bx_i$ denotes the transpose of the $i$-th row vector of $X$, $B_j$ is the $j$-th column vector of $B$, and $\mu>0$ and $\mu_{1,j}>0$ are weighting parameters. However, it should be noted that \eqref{Zou-spca} is indeed still numerically challenging. The combination of a nonsmooth objective and a manifold constraint makes the problem very difficult to solve. Zou \etal \cite{Zou-spca-2006} proposed to solve it using an alternating minimization algorithm (AMA), which updates $A$ and $B$ alternatingly with the other variable fixed as the current iterate. A typical iteration of AMA is as follows
\be\label{zou-ama}
\ba{ll}
A^{k+1} & := \argmin_{A\in\br^{p\times r}} H(A, {B^k}), \st, A^\top A = I_r, \\
B^{k+1} & := \argmin_{B\in\br^{p\times r}} H(A^{k+1},B) + \mu\sum_{j=1}^r\|B_j\|_2^2+\sum_{j=1}^r\mu_{1,j}\|B_j\|_1.
\ea
\ee
The $A$-subproblem in \eqref{zou-ama} is known as a Procrustes rotation problem and has a closed-form solution given by an SVD. The $B$-subproblem in \eqref{zou-ama} is a linear regression problem with an elastic-net regularizer, and it can be solved by many existing solvers such as {elastic net\footnote{R package available from https://cran.r-project.org/web/packages/elasticnet/} \cite{Zou-Hastie-elastic-net-2005}, coordinate descent\footnote{R package available from https://cran.r-project.org/web/packages/glmnet/} \cite{friedman2010regularization} and FISTA \cite{Beck-Teboulle-2009}.}
However, there is no convergence guarantee of AMA \eqref{zou-ama}.
Recently, some new algorithms are proposed in the literature that can solve \eqref{Zou-spca} with guarantees of convergence to a stationary point. We will give a summary of some representative ones in the next section.

We need to point out that there are other ways to formulate Sparse PCA such as the ones in \cite{daspremont-sparsePCA-direct-formulation-2007,daspremont-sparsePCA-JMLR-2008,Ma-SPCA-2011-submit,Lu-Zhang-sparsePCA-MPA-2011,vu2013fantope,d2011identifying,Shen-Huang-spca-2008,Witten2009,Journee2010,Yuan-zhang-jmlr-2013,Moghaddam2006}. We refer interested readers to the recent survey paper \cite{Zou-Xue-spca-survey-2018} for more details on these works on Sparse PCA.
In this paper, we focus on the formulation of \eqref{Zou-spca} to estimate multiple principal components, which is a manifold optimization problem with nonsmooth objective function.

\textbf{Sparse CCA} \cite{wiesel2008greedy,Witten2009,parkhomenko2009sparse,hardoon2011sparse} is proposed to improve the interpretability of CCA, which can be formulated as
\be\label{scca-vector}
\ba{ll}
  \min_{u\in\br^p,v\in\br^q} & -u^\top X^\top Yv + f(u) + g(v) \\
  \st        & u^\top X^\top X u = 1, \ v^\top Y^\top Y v = 1,
\ea
\ee
where $X\in\br^{n\times p}$, $Y\in\br^{n\times q}$, $f$ and $g$ are regularization terms promoting the sparsity of $u$ and $v$, and common choices for them include the $\ell_1$ norm for sparsity and the $\ell_{2,1}$ norm for group sparsity.
When multiple canonical vectors are needed, one can consider the matrix counterpart of \eqref{scca-vector} which can be formulated as
\be\label{scca-mat}
\ba{ll}
 \min_{A\in\br^{p\times r},B\in\br^{q\times r}} & -\Tr(A^\top X^\top YB) +  f(A) + g(B)\\
 \st &  A^\top X^\top XA =I_{r}, \  B^\top Y^\top YB =I_{r},
\ea
\ee
where $r$ is the number of canonical vectors needed. From now on, we call \eqref{scca-vector} the single Sparse CCA model and \eqref{scca-mat} the multiple Sparse CCA model. Moreover, motivated by \cite{gao2017sparse}, in this paper we choose $f$ and $g$ to be the $\ell_{2,1}$ norm to promote the group sparsity of $A$ and $B$ in \eqref{scca-mat}. Specifically, we choose $f(A)=\tau_1\|A\|_{2,1}$, and $g(B)=\tau_2\|B\|_{2,1}$, where the $\ell_{2,1}$ norm is defined as $\|A\|_{2,1}=\sum_{j=1}^p\|A_{j\cdot}\|_2$, and $A_{j\cdot}$ denotes the $j$-th row vector of matrix $A$, and $\tau_1>0$ and $\tau_2>0$ are weighting parameters. In this case, the multiple Sparse CCA \eqref{scca-mat} reduces to
\be\label{scca-mat-L21}
\ba{ll}
 \min_{A\in\br^{p\times r},B\in\br^{q\times r}} & -\Tr(A^\top X^\top YB) + \tau_1\|A\|_{2,1} + \tau_2\|B\|_{2,1}\\
 \st &  A^\top X^\top XA =I_{r}, \  B^\top Y^\top YB =I_{r}.
\ea
\ee
Note that when $r=1$, i.e., when the matrix reduces to a vector, the $\ell_{2,1}$ norm becomes the $\ell_1$ norm of the vector. That is, for vector $u\in\br^p$, $\|u\|_{2,1} = \|u\|_1$, and in this case, the vector Sparse CCA \eqref{scca-vector} reduces to
\be\label{scca-vector-L1}
\ba{ll}
  \min_{u\in\br^p,v\in\br^q} & -u^\top X^\top Yv + \tau_1\|u\|_1 + \tau_2\|v\|_1 \\
  \st        & u^\top X^\top X u = 1, \ v^\top Y^\top Y v = 1.
\ea
\ee
Note that both \eqref{scca-mat-L21} and \eqref{scca-vector-L1} are manifold optimization problems with nonsmooth objectives. Here we assume that both $X^\top X$ are $Y^\top Y$ are positive definite, and we will discuss later the modifications when they are not positive definite.

Manifold optimization recently draws a lot of research attention because of its success in a variety of important applications, including low-rank matrix completion \cite{RTRMC-2011,Vandereycken-matrix-completion-2013}, phase retrieval \cite{Boumal-phase-retrieval-2018,Sun-Ju-geometric-phase-retrieval-2018}, phase synchronization \cite{Boumal-phase-synchronization-2016,Liu-generalized-power-phase-synchronization-2017}, blind deconvolution \cite{Huang-2018}, and dictionary learning \cite{Sra-Riemannian-dictionary-learning-2016,Sun-dictionary-recovery-sphere-2017}. 
Most existing algorithms for solving manifold optimization problems rely on the smooothness of the objective, see the recent monograph by Absil et al. \cite{{Absil2009}}. Studies on manifold optimization problems with nonsmooth objective such as \eqref{Zou-spca}, \eqref{scca-mat-L21}, and \eqref{scca-vector-L1} have been very limited.
This urges us to study efficient algorithms that solve manifold optimization problems with nonsmooth objective, and this is the main focus of this paper. 

The rest of this paper is organized as follows. We review existing methods for Sparse PCA and Sparse CCA in Section \ref{sec:existing}. We propose a unified alternating manifold proximal gradient method with provable convergence guarantees for solving both Sparse PCA and Sparse CCA in Section \ref{sec:algorithm}. The numerical performance is demonstrated in Section \ref{sec:num}. We provide preliminaries on manifold optimization and details of the global convergence analysis of our proposed method in the Appendix.

\section{Existing Methods}\label{sec:existing}

Before proceeding, we review existing methods for solving Sparse PCA \eqref{Zou-spca} in Section \ref{sec:alg-spca} and for solving Sparse CCA \eqref{scca-vector} and \eqref{scca-mat} in Section \ref{sec:alg-scca}.

\subsection{Solving Sparse PCA}\label{sec:alg-spca}

For Sparse PCA \eqref{Zou-spca}, other than the AMA algorithm suggested in the original paper \cite{Zou-spca-2006}, there exist some other efficient algorithms for solving this problem. We now give a brief review of these works. We first introduce two powerful optimization algorithms for solving nonconvex problems: proximal alternating minimization (PAM) algorithm \cite{attouch2010proximal} and proximal alternating linearization method (PALM) \cite{bolte2014proximal}. Surprisingly, it seems that these two methods have not been used to solve \eqref{Zou-spca} yet. We now briefly describe how these two methods can be used to solve \eqref{Zou-spca}. PAM for \eqref{Zou-spca} solves the following two subproblems in each iteration:
\be\label{PAM}
\bad
A_{k+1} & := \argmin_{A} H(A,B_k) + \frac{1}{2t_1}\|A-A_k\|_F^2, \st, \ A^\top A = I_r, \\
B_{k+1} & := \argmin_{B} H(A_{k+1},B) + \mu\sum_{j=1}^r\|B_j\|_2^2+\sum_{j=1}^r\mu_{1,j}\|B_j\|_1 + \frac{1}{2t_2}\|B-B_k\|_F^2,
\ead
\ee
where $t_1>0$, $t_2>0$ are stepsizes. Note that in each subproblem, PAM minimizes the objective function with respect to one variable by fixing the other, and a proximal term is added for the purpose of convergence guarantee. It is shown in \cite{attouch2010proximal} that the sequence of PAM converges to a critical point of \eqref{Zou-spca} under the assumption that the objective function satisfies the Kurdyka-{{\L}}ojasiewicz (KL) inequality\footnote{Without KL inequality, only subsequence convergence is obtained.}. We need to point out that the only difference between PAM \eqref{PAM} and the AMA \eqref{zou-ama} is the proximal terms, which together with the KL inequality helps establish the convergence result. Note that the $A$-subproblem in \eqref{PAM} corresponds to the reduced rank procrustes rotation and can be solved by an SVD. The $B$-subproblem in \eqref{PAM} is a Lasso type problem and can be solved efficiently by first-order methods such as FISTA or block coordinate descent. A better algorithm that avoids iterative solver for the subproblem is PALM, which linearizes the quadratic functions in the subproblems of \eqref{PAM}. A typical iteration of PALM is:
\be\label{PALM}
\bad
A_{k+1} & := \argmin_{A} \ \langle \nabla_A H(A_k,B_k), A\rangle + \frac{1}{2t_1}\|A-A_k\|_F^2,  \st, \ A^\top A = I_r, \\
B_{k+1} & := \argmin_{B} \ \langle \nabla_B H(A_{k+1},B_k), B\rangle + \frac{1}{2t_2}\|B-B_k\|_F^2+ \mu\sum_{j=1}^r\|B_j\|_2^2+\sum_{j=1}^r\mu_{1,j}\|B_j\|_1,
\ead
\ee
where $\nabla_A H$ and $\nabla_B H$ denote the gradient of $H$ with respect to $A$ and $B$, respectively.
The two subproblems in \eqref{PALM} are easier to solve than the ones in \eqref{PAM} because they both admit closed-form solutions. In particular, the solution of the $A$-subproblem in \eqref{PALM} corresponds to the projection onto the orthogonality constraint, which is given by an SVD; the solution of the $B$-subproblem in \eqref{PALM} is given by the $\ell_1$ soft-thresholding operation. It is shown in \cite{bolte2014proximal} that the sequence of PALM converges to a critical point of \eqref{Zou-spca} under the assumption that the objective function satisfies the Kurdyka-{{\L}}ojasiewicz inequality.
Recently, Erichson \etal \cite{erichson2018sparse} proposed a projected gradient method based on variable projection (VP) for solving \eqref{Zou-spca}. Though the motivation of this algorithm is different, it can be viewed as a variant of PAM and PALM. Roughly speaking, the VP algorithm combines the $A$-subproblem (without the proximal term) in \eqref{PAM} and the $B$-subproblem in \eqref{PALM}. That is, it updates the iterates as follows:
\be\label{VP}
\bad
A_{k+1} & := \argmin_{A} \ H(A,B_k), \st, \ A^\top A = I_r, \\
B_{k+1} & := \argmin_{B} \ \langle \nabla_B H(A_{k+1},B_k), B\rangle + \frac{1}{2t_2}\|B-B_k\|_F^2+ \mu\sum_{j=1}^r\|B_j\|_2^2+\sum_{j=1}^r\mu_{1,j}\|B_j\|_1.
\ead
\ee
Note that the difference of PALM \eqref{PALM} and VP \eqref{VP} lies in the $A$-subproblem. The $A$-subproblem linearizes the quadratic function $H(A,B_k)$ in \eqref{PALM} but not in \eqref{VP}. This does not affect much the performance of the algorithms because in this specific problem the $A$-subproblems correspond to an SVD in both algorithms. It is shown in \cite{erichson2018sparse} that VP \eqref{VP} converges to a stationary point of \eqref{Zou-spca}. Another recent work that can solve \eqref{Zou-spca} is the ManPG (manifold proximal gradient method) algorithm proposed by Chen \etal \cite{chen2018proximal}. We will discuss it in more details later as it is closely related to the algorithm we propose in this paper.
For other algorithms for solving Sparse PCA, we refer the interested readers to the recent survey paper \cite{Zou-Xue-spca-survey-2018} for more details. 

%


\subsection{Solving Sparse CCA}\label{sec:alg-scca}

Chen \etal \cite{chen2013sparse} proposed a CAPIT (standing for canonical correlation analysis via precision adjusted iterative thresholding) algorithm for solving the single Sparse CCA \eqref{scca-vector-L1}. The CAPIT algorithm alternates between an iterative thresholding step and a power method step, to deal with the sparsity regularization and orthogonality constraints respectively. 
The CoLaR (standing for Convex program with group-Lasso Refinement) method proposed by Gao \etal \cite{gao2017sparse} targets to solve the multiple Sparse CCA \eqref{scca-mat}. CoLaR is a two-stage algorithm. In the first stage, a convex relaxation of \eqref{scca-mat} based on the matrix lifting technique is solved. In the second stage, the solution obtained from the first stage is refined by solving a group Lasso type problem. In \cite{wiesel2008greedy}, Wiesel \etal proposed a greedy approach for solving \eqref{cca-single} with cardinality constraints on $u$ and $v$. There is no convergence guarantee of this greedy approach due to the challenges posed by the combinatorial nature of the cardinality function. Recently, Suo \etal \cite{suo2017sparse} proposed an alternating minimization algorithm (AMA) for solving the single Sparse CCA \eqref{scca-vector-L1}, which solves two subproblems in each iteration by solving \eqref{scca-vector-L1} with respect to $u$ (resp. $v$) with $v$ (resp. $u$) fixed as $v^k$ (resp. $u^k$). The subproblems were then solved by a linearized ADMM (alternating direction method of multipliers) algorithm.
We need to point out that none of these algorithms for Sparse CCA has a convergence guarantee.
There exist some other methods for Sparse CCA (see, e.g., \cite{Witten2009,hardoon2011sparse}), but we omit their details here because they are not directly related to \eqref{scca-vector} and \eqref{scca-mat}. We also point out that, the PAM, PALM and VP algorithms discussed in Section \ref{sec:alg-spca} do not apply to Sparse CCA \eqref{scca-vector} and \eqref{scca-mat} because they all result in complicated subproblems.
For instance, to apply PALM to \eqref{scca-vector}, one needs to compute the proximal mapping of $f(u) + \iota(u^\top X^\top Xu=1)$, which does not admit a closed-form solution and is thus computationally expensive, where $\iota(\cdot)$ denotes the indicator function.

\section{A Unified A-ManPG Algorithm}\label{sec:algorithm}

In this section, we give a unified treatment for solving Sparse PCA \eqref{Zou-spca} and Sparse CCA \eqref{scca-vector} and \eqref{scca-mat}, and introduce our alternating manifold proximal gradient algorithm (A-ManPG) for solving them. We first note that both Sparse PCA \eqref{Zou-spca} and Sparse CCA \eqref{scca-vector} and \eqref{scca-mat} are special cases of the following problem:
\be\label{general_prob}
\min F(A,B): = H(A,B) + f(A) + g(B), \  \st \ A\in \M_1, B\in \M_2,
\ee
where $H(A,B)$ is a smooth function of $A,B$ with a Lipschitz continuous gradient, $f(\cdot)$ and $g(\cdot)$ are lower semi-continuous convex functions with relatively easy proximal mappings, and $\M_1,\M_2$ are two sub-manifolds embedded in the Euclidean space. The Sparse PCA \eqref{Zou-spca} is in the form of \eqref{general_prob} with
$H(A,B)=\sum_{i=1}^n \|\bx_i-AB^\top \bx_i\|_2^2$, $f(A) \equiv 0$, $g(B)=\mu\sum_{j=1}^r\|B_j\|_2^2+\sum_{j=1}^r\mu_{1,j}\|B_j\|_1$, $\M_1=\{A\mid A^\top A = I_r\}$ (the Stiefel manifold) and $\M_2=\br^{p\times r}$. The single Sparse CCA \eqref{scca-vector} is in the form of \eqref{general_prob} with
$H(u,v)=-u^\top X^\top Yv$, $\M_1=\{u\mid u^\top X^\top X u = 1\}$, $\M_2=\{v\mid v^\top Y^\top Y v = 1\}$. The multiple Sparse CCA \eqref{scca-mat} is in the form of \eqref{general_prob} with $H(A,B)=-\Tr(A^\top X^\top YB)$, $\M_1= \{A\mid A^\top X^\top XA =I_{r}\}$, $\M_2 = \{B\mid B^\top Y^\top YB =I_{r}\}$. Note that here in Sparse CCA \eqref{scca-vector} and \eqref{scca-mat} we assumed that both $X^\top X$ and $Y^\top Y$ are positive definite to guarantee that $\M_1$ and $\M_2$ are sub-manifolds. If they are not positive definite, we can always add a small perturbation to make them so. These manifolds used in Sparse CCA \eqref{scca-vector} and \eqref{scca-mat} are generalized Stiefel manifolds.

The ManPG algorithm proposed by Chen \etal \cite{chen2018proximal} can be applied to solve \eqref{general_prob}. In each iteration, ManPG linearizes $H(A,B)$ and solves the following convex subproblem:
\be\label{sub_manpg}
\ba{ll}
\displaystyle\min_{D^A,D^B} & \displaystyle\left\langle \begin{pmatrix} \nabla_{A} H(A_k,B_k) \\ \nabla_{B} H(A_k,B_k)\end{pmatrix}, \begin{pmatrix} D^A \\ D^B \end{pmatrix} \right\rangle + \frac{1}{2t_1}\normfro{D^A}^2+\frac{1}{2t_2}\normfro{D^B}^2+f(A_k+D^A)+g(B_k+D^B),\\
\st        & D^A\in \T_{A_k}\M_1, D^B\in\T_{B_k}\M_2,
\ea
\ee
where $t_1 < 1/L$, $t_2 < 1/L$ and $L$ is the Lipschitz constant of $\nabla H(A,B)$ on the tangent space $\T_{A_k}\M_1 \times \T_{B_k}\M_2$. For the Stiefel manifold $\M=\{A\mid A^\top A = I_r\}$, its tangent space is given by $\T_{A}\M = \{D\mid D^\top A + A^\top D = 0\}$, and for the generalized Stiefel manifold $\M=\{A\mid A^\top M A = I_r\}$, its tangent space is given by $\T_{A}\M = \{D\mid D^\top MA + A^\top MD = 0\}$. Note that \eqref{sub_manpg} is actually separable for $D^A$ and $D^B$ and thus reduces to two subproblems for $D^A$ and $D^B$ respectively. As a result, ManPG \eqref{sub_manpg} can be viewed as a Jacobi-type algorithm in this case, as it computes $D^A$ and $D^B$ in parallel. We found from our numerical experiments that the algorithm converges faster if $D^A$ and $D^B$ are computed in a Gauss-Seidel manner. This leads to the following updating scheme, which is the basis of our alternating manifold proximal gradient (A-ManPG) algorithm:
\be\label{sub_amanpg}
\bad
D^A_k & := \argmin_{D^A} \ \inp{\nabla_A H(A_k,B_k)}{D^A} +f(A_k+D^A) + \frac{1}{2t_1}\normfro{D^A}^2, \ \st \ D^A\in \T_{A_k}\M_1, \\
D^B_k & := \argmin_{D^B} \ \inp{\nabla_B H(A_{k+1},B_k)}{D^B}+g(B_k+D^B) + \frac{1}{2t_2}\normfro{D^B}^2, \ \st \ D^B\in \T_{B_k}\M_2,
\ead
\ee
where $A_{k+1}$ is obtained via a retraction operation (see Algorithm \ref{alg:amgdpgd}), $t_1 < 1/L_A$, $t_2 < 1/L_B$ and $L_A$ and $L_B$ are Lipschitz constants of $\nabla_A H(A,B_k)$ and $\nabla_B H(A_{k+1},B)$ on tangent spaces $\T_{A_k}\M_1$ and $\T_{B_k}\M_2$, respectively.
The Gauss-Seidel type algorithm A-ManPG usually performs much better than the Jacobi-type algorithm ManPG, because the Lipschitz constants are smaller and thus larger step sizes are allowed. The details of the A-ManPG algorithm are described in Algorithm \ref{alg:amgdpgd}.

\begin{algorithm}[ht]
    \caption{Alternating Manifold Proximal Gradient Method (A-ManPG)  }\label{alg:amgdpgd}
    \begin{algorithmic}[1]
        \STATE{Input: Initial point $(A_0,B_0)$, parameters $\delta\in(0,1)$, $\gamma\in(0,1)$, step sizes $t_1$ and $t_2$.}
        \FOR{$k=0,1,\ldots,$}
        \STATE{Solve the $A$-subproblem in \eqref{sub_amanpg} to obtain $D^A_k$.}
        \STATE{Set $\alpha_1=1$.} 
        \WHILE{$F(R_{A_k}(\alpha_1 D^A_k) ,B_k)> F(A_k,B_k)-\delta \alpha_1 \normfro{D^A_k}^2$}
        \STATE{$\alpha_1= \gamma\alpha_1$}
        \ENDWHILE
        \STATE{Set $A_{k+1}=R_{A_k}(\alpha_1 D^A_k)$.}
        \STATE{Solve the $B$-subproblem in \eqref{sub_amanpg} to obtain $D^B_k$.}
        \STATE{Set $\alpha_2=1$.} 
        \WHILE{$F(A_{k+1},R_{B_k}(\alpha_2 D^B_k))> F(A_{k+1},B_k)-\delta \alpha_2 \normfro{D^B_k}^2$}
        \STATE{$\alpha_2= \gamma\alpha_2$}
        \ENDWHILE
        \STATE{Set $B_{k+1}=R_{B_k}(\alpha_2 D^B_k)$}.
        \ENDFOR
    \end{algorithmic}
\end{algorithm}

\begin{remark}
Note that the iterates $A_k$ and $B_k$ are kept on the manifolds through the retraction operations $R_A$ and $R_B$. There exist many choices for the retraction operations, and in Algorithm \ref{alg:amgdpgd}, we did not specify which ones to use. We discuss common retractions for Stiefel manifold and generalized Stiefel manifold in the Appendix. In our numerical experiments in Section \ref{sec:num}, we chose polar decomposition as the retraction. Lines 4-7 and 9-13 in Algorithm \ref{alg:amgdpgd} are backtracking line search procedures. These are necessary to guarantee that the objective function has a sufficient decrease in each iteration, which is needed for the convergence analysis (see the Appendix).
\end{remark}


From Lemma \ref{first_order_opt} (see Appendix), we know that $D^A_k=0$ and $D^B_k=0$ imply that $(A_k,B_k)$ is a stationary point for problem \eqref{general_prob}. As a result, we can define an $\epsilon$-stationary point of \eqref{general_prob} as follows.

\begin{definition}\label{def-epsilon-stationary}
    $(A_k,B_k)$ is called an $\epsilon$-stationary point of \eqref{general_prob} if $D^A_k$ and $D^B_k$ returned by \eqref{sub_amanpg} satisfy $(\normfro{D^A_k}^2+\normfro{D^B_k}^2) \leq \epsilon^2$.
\end{definition}

We have the following convergence results for the A-ManPG algorithm (Algorithm \ref{alg:amgdpgd}).

\begin{theorem}\label{thm:complexity}
    Any limit point of the sequence $\{(A_k,B_k)\}$ generated by Algorithm \ref{alg:amgdpgd} is a stationary  point of problem \eqref{general_prob}. Furthermore, Algorithm \ref{alg:amgdpgd} returns an $\epsilon$-stationary point $(A_k,B_k)$ in at most $(F(A_0,B_0)-F^*)/((\bar{\beta}_1 + \bar{\beta}_2)\epsilon^2)$ iterations, where $F^*$ denotes a lower bound of the optimal value of \eqref{general_prob}, $\bar{\beta}_1>0$ and $\bar{\beta}_2>0$ are constants. 
\end{theorem}

\begin{proof}
The proof is given in the Appendix.
\end{proof}

\subsection{Semi-Smooth Newton Method for the Subproblems}

The main computational effort in each iteration of Algorithm \ref{alg:amgdpgd} is to solve the two subproblems in \eqref{sub_amanpg}. For Stiefel manifold and generalized Stiefel manifold, the two subproblems in \eqref{sub_amanpg} are both equality-constrained convex problems, given that both $f$ and $g$ are convex functions. {Note that if $f$ (resp.  $g$) vanishes, the $A$-subproblem (resp. $B$-subproblem) becomes the projection onto the tangent space of $\M_1$ (resp. $\M_2$), which reduces to Riemannian gradient step and can be easily done. Here we discuss the general case where $f$ and $g$ do not vanish. In this case,} we found that a regularized semi-smooth Newton (SSN) method \cite{Xiao2016} is very suitable for solving this kind of problems. The notion of semi-smoothness was originally introduced by Mifflin \cite{Mifflin-1977} for real-valued functions and extended to vector-valued mappings by Qi and Sun \cite{Qi-Sun-1993}. A pioneer work on the SSN method was due to Solodov and Svaiter \cite{Solodov1998}, where the authors proposed a globally convergent Newton's method by exploiting the structure of monotonicity, {and local superlinear rate was established under the conditions that  generalized Jacobian is semi-smooth and non-singular at the global optimal point. The convergence rate is extended in \cite{Zhou2005} to the setting where the generalized Jacobian is not necessarily non-singular.} Recently, SSN has received significant attention due to its success in solving structured convex problems to high accuracy. In particular, it has been successfully applied to solving SDP \cite{ZhaoSunToh2008,Sun-sdpnal+}, Lasso \cite{Sun-lasso-2018}, nearest correlation matrix estimation \cite{Qi-Sun-NCM-IMA}, clustering \cite{Wang-Sun-Toh-2009}, sparse inverse covariance selection \cite{Sun-group-lasso}, and composite convex minimization \cite{Xiao2016}.

We now describe how to apply the regularized SSN method in \cite{Xiao2016} to solve the subproblems in \eqref{sub_amanpg}. For brevity, we only focus on the $A$-subproblem with $\M_1=\{A\mid A^\top X^\top X A=I_r\}$ {and $f(A)= \tau_1\normtwo{A}_{2,1}$} as used in \eqref{scca-mat-L21}. For the ease of notation, we denote $t=t_1$, $D=D^A$, $M:=X^\top X$, $h(A) := H(A,B_k)$. In this case, the $A$-subproblem in \eqref{sub_amanpg} reduces to
\be\label{A-sub} D_k := \argmin_{D} \ \inp{\nabla h(A_k)}{D} +f(A_k+D) + \frac{1}{2t}\normfro{D}^2, \ \st \ D^\top MA_k + A_k^\top MD = 0. \ee
By associating a Lagrange multiplier $\Lambda$ to the linear equality constraint, the Lagrangian function of \eqref{A-sub} can be written as
\be\label{Lag-func}
\LCal(D;\Lambda) = \inp{\nabla h(A_k)}{D} +\frac{1}{2t}\normtwo{D}_F^2+f(A_k+D) - \inp{D^\top MA_k+A_k^\top M D}{\Lambda},
\ee
and the Karush-Kuhn-Tucker (KKT) system of \eqref{A-sub} is given by
\be\label{tangent-subproblem-kkt}
0 \in \partial_D \LCal(D;\Lambda), \mbox{ and } D^\top MA_k+A_k^\top M D = 0.
\ee
The first condition in \eqref{tangent-subproblem-kkt} implies that $D$ can be computed by
\be\label{compute-D}
D(\Lambda) = \prox_{tf}(B(\Lambda))-A_k, \mbox{ with } B(\Lambda) = A_k-t(\nabla h(A_k) -2 MA_k\Lambda),
\ee
where $\prox_f(A)$ denotes the proximal mapping of function $f$ at point $A$.
By substituting \eqref{compute-D} into the second condition in \eqref{tangent-subproblem-kkt}, we obtain that $\Lambda$ satisfies
\begin{equation}\label{sub_VI}
E(\Lambda) := D(\Lambda)^\top M A_k + A_k^\top MD(\Lambda) = 0,
\end{equation}
and thus the problem reduces to finding a root of function $E$. 
Since $E$ is a monotone operator (see \cite{chen2018proximal}) and the proximal mapping of the $\ell_2$ norm is semi-smooth\footnote{{The definition is given in the Appendix. The proximal mapping of $\ell_p (p\geq 1)$ norm is strongly semi-smooth \cite{facchinei2007finite,ulbrich2011semismooth}. From \cite[Prop. 2.26]{ulbrich2011semismooth}, if $F: V\rightarrow \R^m$ is a piecewise $\mathcal{C}^1$ (piecewise smooth) function, then $F$ is semi-smooth. If $F$ is a piecewise $\mathcal{C}^2$ function, then $F$ is strongly semi-smooth. It is known that proximal mappings of many interesting functions are piecewise linear or piecewise smooth.}}, we can apply SSN to find the zero of $E$. The SSN method requires to compute the generalized Jacobian of $E$, and in the following we show how to compute it.
We first derive the vectorization of $E(\Lambda)$.
\[
\begin{aligned}
\vet(E(\Lambda))=&((MA_k)^\top\otimes I_r)\vet{(D(\Lambda)^\top)}+(I_r\otimes (MA_k)^\top)K_{rn} \vet{(D(\Lambda)^\top)}\\
=&(I_{r^2}+K_{rr})((MA_k)^\top\otimes I_r)[\prox_{tf(\cdot)}(\vet((MA_k)^\top - t\nabla h(A_k)^\top) \\
 &+2t((MA_k) \otimes I_r)\vet(\Lambda))-\vet(X_k^\top)],
\end{aligned}
\]
where $K_{rn}$ and $K_{rr}$ denote the commutation matrices. We define the following matrix
\[\G(\vet(\Lambda))=t ((MA_k)^\top\otimes I_r)\J(y)|_{y=\vet(B(\Lambda)^\top)} ((MA_k)\otimes I_r),\]
where $\otimes$ denotes the Kronecker product, and $\J(y)$ is the generalized Jacobian of $\prox_{ tf(\cdot)}(y)$ which is defined as follows:
\[\J(y)|_{y=\vet(B(\Lambda)^\top)} = \begin{pmatrix}
    \Delta_1 &  & \\
    & \ddots & \\
    &  &  \Delta_p
\end{pmatrix},
\]
where the matrices $\Delta_j,j=1,\ldots,p$ are defined as
\[
\Delta_j=\begin{cases}
I_r-\frac{\tau_1t}{\|b_j\|_2 }(I_r - \frac{b_jb_j^\top}{\|b_j\|_2}), & \text{ if } \|b_j\|_2>t\tau_1\\
\gamma \frac{b_j b_j^\top}{(t\tau_1)^2}: \gamma\in[0,1], & \text{ if } \|b_j\|_2=t\tau_1\\
0, & \text{ otherwise,}
\end{cases}
\]
where $b_j$ is the $j$-th column of matrix $B(\Lambda)^\top$. It is then easy to see that $\G(\vet(\Lambda))$ is positive-semidefinite\footnote{We say a matrix $A$ is positive semi-definite if $A+A^\top$ is positive semi-definite.}. From \cite[Example 2.5]{hiriart1984generalized}, we know that
$\G(\vet(\Lambda))\xi = \partial \vet(E(\vet(\Lambda))\xi, \ \forall \xi\in \R^{r^2}$.
So, $\G(\vet(\Lambda))$ serves as an alternative of $\partial\vet(E(\vet(\Lambda)))$. It is known that the global convergence of regularized SSN is guaranteed if any element of $\G(\vet(\Lambda))$ is positive semi-definite \cite{Xiao2016}. For local convergence rate, one needs more conditions on  $\partial\vet(E(\vet(\Lambda)))$. We refer to \cite{Xiao2016} for more details. {Note that since $\Lambda$ is a symmetric matrix, we can work with the lower triangular part of $\Lambda$ only and remove the duplicated entries in the upper triangular part. To do so, we use $\svec(\Lambda)$ to denote the $\frac{1}{2}r(r+1)$-dimensional vector obtained from $\vvec(\Lambda)$ by eliminating all super-diagonal elements of $\Lambda$. It is known that there exists a unique $r^2\times \frac{1}{2}r(r+1)$ matrix $U_r$, which is called the duplication matrix \cite[Ch 3.8]{Magnus1988}, such that $U_r \svec(\Lambda)=\vvec(\Lambda)$. The Moore-Penrose inverse of $U_r$ is $U_r^+=(U_r^\top U_r)^{-1}U_r^\top$ and it satisfies $U_r^+ \vet(\Lambda)=\svec(\Lambda)$. Note that both $U_r$ and $U_r^+$ have only $r^2$ nonzero elements.}
The {alternative of} generalized Jacobian of $\svec(E(U_r\svec(\Lambda)))$ is given by
\be\label{general-symmetric-1-S}
\GG(\svec(\Lambda))=tU_r^+ \G(\vet(\Lambda))U_r=4tU_r^+((MA_k)^\top\otimes I_r)\J(y)|_{y=\vet(B(\Lambda)^\top)} ((MA_k)\otimes I_r)U_r,
\ee
where we used the identity $K_{rr}+I_{r^2}=2U_rU_r^+$.
Therefore, \eqref{general-symmetric-1-S} can be simplified to
\be\label{general_G-S}
\bad
 & G(\svec(\Lambda))\\
=&
4tU_r^+((MA_k)^\top\otimes I_r)\begin{pmatrix}
    \Delta_1 &  & \\
    & \ddots & \\
    &  &  \Delta_p
\end{pmatrix} ((MA_k)\otimes I_r)U_r\\
=&4tU_r^+
\begin{pmatrix}
    \sum_{j=1}^{p} (MA_k)_{j1}^2\Delta_j & \sum_{j=1}^{p} (MA_k)_{j1}(MA_k)_{j2}\Delta_j  &\cdots &\sum_{j=1}^{p} (MA_k)_{j1}(MA_k)_{jr}\Delta_j  \\
    \\
    \sum_{j=1}^{p} (MA_k)_{j2}(MA_k)_{j1}\Delta_j & \sum_{j=1}^{p} (MA_k)_{j2}^2\Delta_j  &\cdots &\sum_{j=1}^{p} (MA_k)_{j2}(MA_k)_{jr}\Delta_j  \\
    \vdots & \vdots & \vdots & \vdots\\
    \sum_{j=1}^{p} (MA_k)_{jr}(MA_k)_{j1}\Delta_j&\sum_{j=1}^{p} (MA_k)_{jr}(MA_k)_{j2}\Delta_j  &\cdots &\sum_{j=1}^{p} (MA_k)_{jr}^2\Delta_j
\end{pmatrix} U_r.
\ead
\ee
The regularized SSN in \cite{Xiao2016} first computes the Newton's direction $d_k$ by solving
\be\label{newton-direction}
(\GG(\svec(\Lambda_k)) + \eta I)d = -\svec(E(\svec(\Lambda_k))),
\ee
where $\eta>0$ is a regularization parameter. {Note that $\eta$ is necessary here because $ G(\svec(\Lambda))$ could be singular if $\Delta_j=0$ for some $j$.}
$\Lambda_k$ is then updated by
\[\svec(\Lambda_{k+1}) = \svec(\Lambda_k) + d_k.\]
The regularized SSN proposed in \cite{Xiao2016} combines some other techniques to make the algorithm more robust, but we omit the details here. We refer to \cite{Xiao2016} for more details on this algorithm.

\section{Numerical Experiments} \label{sec:num}

\subsection{Sparse PCA}
In this section, we apply our algorithm A-ManPG to solve Sparse PCA \eqref{Zou-spca}, and compare its performance with three existing methods: AMA \cite{Zou-spca-2006}, PALM \cite{bolte2014proximal} and VP \cite{erichson2018sparse}. The details of the parameter settings of these algorithms are given below.
\begin{itemize}
\item AMA \eqref{zou-ama}: FISTA \cite{Beck-Teboulle-2009} is used to solve the $B$-subproblem. Maximum iteration number is set to 1000.
\item PALM \eqref{PALM}: $t_1:=1$, $t_2:=1/(2\lambda_{\max}(X^\top X))$. Maximum iteration number is set to 10000.
\item VP \eqref{VP}: $t_2:=1/(2\lambda_{\max}(X^\top X))$. Maximum iteration number is set to 10000.
\item A-ManPG: $t_1= 100/p, t_2:=1/(2\lambda_{\max}(X^\top X))$. Maximum iteration number is set to 10000.
\end{itemize}



The algorithms are terminated using the following criteria. First, we use PALM as a base line, and we denote the objective function value in \eqref{Zou-spca}  as $F(A,B)$, $\ie, F(A,B)= H(A,B) + \mu\sum_{j=1}^r\|B_j\|^2+\sum_{j=1}^r\mu_{1,j}\|B_j\|_1 $. We terminate PALM when we find that
\be\label{decrease-H} | F_{PALM}(A_{k+1},B_{k+1}) - F_{PALM}(A_k,B_k) | < 10^{-5}. \ee
We then terminate AMA, A-ManPG, and VP when their objective function value is smaller than $F_{PALM}$ and the change of their objective values in two consecutive iterations is less than $10^{-5}$.

We generate the data matrix $X$ in the following manner. First, the entries of $X$ are generated following standard normal distribution $\mathcal{N}(0,1)$. The columns of $X$ are then centered so that the columns have zero mean and they are then scaled by dividing the largest $\ell_2$ norm of the columns. We report the comparison results of the four algorithms in Tables \ref{tab:spca_rand_1} and \ref{tab:spca_rand_2} where $r=6$ for all cases. In particular, Table \ref{tab:spca_rand_1} reports the results for $n<p$, and we tested $\mu=1$ and $\mu=10$, because it is suggested in \cite{Zou-spca-2006} that $\mu$ should be relatively large in this case. Table \ref{tab:spca_rand_2} reports the results for $n>p$, and we set $\mu=10^{-6}$, because it is suggested in \cite{Zou-spca-2006} that $\mu$ should be sufficiently small in this case. In these tables, CPU times are in seconds, and 'sp' denotes the percentage of zero entries of matrix $B$. From Tables \ref{tab:spca_rand_1} and \ref{tab:spca_rand_2} we see that the four algorithms generated solutions with similar objective function value $F(A,B)$ and similar sparsity 'sp'. In terms of CPU time, AMA is the slowest one, and the other three are comparable and are all much faster than AMA. This is due to the reason that AMA needs an iterative solver to solve the $B$-subproblem, which is time-consuming in practice.




\input{table-spca-1.tex}


\input{table-spca-2.tex}

\subsection{Sparse CCA: Vector Case}\label{scca:vector}

In this section, we report the numerical results of A-ManPG for solving the single Sparse CCA \eqref{scca-vector-L1}, and compare its performance with a recent approach proposed by Suo \etal \cite{suo2017sparse}: AMA+LADMM. {More specifically, AMA+LADMM aims at solving the relaxation of \eqref{scca-vector-L1} as follows
\be\label{scca-vector-L1-relax}
\ba{ll}
\min_{u\in\br^p,v\in\br^q} & -u^\top X^\top Yv + \tau_1\|u\|_1 + \tau_2\|v\|_1 \\
\st        & u^\top X^\top X u \leq 1, \ v^\top Y^\top Y v \leq  1.
\ea
\ee
AMA+LADMM works in the following manner. In the $k$-th iteration, $v$ is fixed as $v_k$ and the following convex problem of $u$ is solved:
\be\label{scca-ama-u} u_{k+1} := \argmin_u -u^\top X^\top Yv_k + \tau_1\|u\|_1, \quad \st \quad u^\top X^\top X u \leq 1. \ee
Then, $u$ is fixed as $u_{k+1}$ and the following convex problem of $v$ is solved
\be\label{scca-ama-v} v_{k+1} := \argmin_v -u_{k+1}^\top X^\top Yv + \tau_2\|v\|_2, \quad \st \quad v^\top Y^\top Y v \leq 1. \ee
The linearized ADMM (LADMM) is used to solve the two convex subproblems \eqref{scca-ama-u} and \eqref{scca-ama-v}.
}

We generate the data following the same manner as in \cite{suo2017sparse}. Specifically, two data sets $X\in\br^{n\times p}$ and $Y\in\br^{n\times q}$ are generated from the following model:
\be\label{data-normal}
\begin{pmatrix} x \\ y \end{pmatrix} \sim \mathcal{N} \left(\begin{pmatrix} 0 \\ 0 \end{pmatrix}, \begin{pmatrix} \Sigma_x & \Sigma_{xy} \\ \Sigma_{yx} & \Sigma_y \end{pmatrix}\right),
\ee
where $\Sigma_{xy} = \hat{\rho}\Sigma_x \hat{u}\hat{v}^\top \Sigma_y$, $\hat{u}$ and $\hat{v}$ are the true canonical vectors, and $\hat{\rho}$ is the true canonical correlation. In our numerical tests, $\hat{u}$ and $\hat{v}$ are generated randomly such that they both have $5$ non-zero entries and the nonzero coordinates are set at the $\{1,6,11,16,21\}$-th coordinates. The nonzero entries are obtained from normalizing (with respect to $\Sigma_x$ and $\Sigma_y$) random numbers drawn from the uniform distribution on the finite set $\{-2,-1,0,1,2\}$. We set $\hat{\rho}=0.9$ in all tests.
We tested three different ways to generate the covariance matrices $\Sigma_x$ and $\Sigma_y$.
\begin{itemize}
\item Identity matrices: $\Sigma_x = I_p$, $\Sigma_y = I_q$.
\item Toeplitz matrices: $[\Sigma_x]_{ij} = 0.9^{\abs{i-j}}$ and $[\Sigma_y]_{ij} = 0.9^{\abs{i-j}}$.
\item Sparse inverse matrices: $[\Sigma_x]_{ij}={\sigma^0_{ij}}/{\sqrt{\sigma^0_{ii}\sigma_{jj}^0}}, $ where $\Sigma^0=(\sigma_{ij}^0)=\Omega^{-1}$ and
    \[\Omega_{ij}=\iota_{i=j} +0.5\times{1}_{\abs{i-j}=1}+0.4\times\iota_{\abs{i-j}=2}.\]
\end{itemize}
 $\Sigma_y$ is generated in the same way. The matrices $X$ and $Y$ are both divided by $\sqrt{n-1}$ such that $X^\top Y$ is the estimated covariance matrix.
Note that if $n<p$ or $n<q$, the covariance matrix $X^\top X$ or $Y^\top Y$ is not positive definite. In this case, we replace $X^\top X$ by $(1-\alpha)X^\top X +\alpha I_p$ and $Y^\top Y$ by $(1-\alpha)Y^\top Y+\alpha I_q$ in the constraints of \eqref{scca-vector-L1}, so that we can still keep them as manifold constraints. In our experiments, we chose $\alpha = 10^{-4}$. The same as \cite{suo2017sparse}, we define two loss functions 'lossu' and 'lossv' to measure the distance between the ground truth $(\hat{u},\hat{v})$ and estimation $(u,v)$:
\[\text{lossu}=2(1-\abs{\hat{u}^\top u}), \quad\text{lossv}=2(1-\abs{\hat{v}^\top v}),\]
where $(u,v)$ is the iterate returned by the algorithm. Moreover, the following procedure for initialization suggested in \cite{suo2017sparse} is adopted. First, we truncate the matrix $X^\top Y$ by soft-thresholding its small elements to be $0$ and denote the new matrix $S_{xy}$. {More specifically, we set the entries of $S_{xy}$  to zeros if their magnitudes are smaller than the largest magnitude of the diagonal elements. Secondly, we compute the singular vectors $u_0$ and $v_0$ {corresponding to the largest singular value} of $S_{xy}$ and then normalize them using $u_0:=u_0/\sqrt{u_0^\top X^\top Xu_0}$ and $v_0:=v_0/\sqrt{v_0^\top Y^\top Yv_0}$ as initialization of $u$ and $v$. We set $\tau_1 = \tau_2 =\half b\sqrt{\log(p+q)/n}$ in \eqref{scca-vector-L1} where $b$ was set to $b=\{1,1.2,1.4,1.6\}$. We report the best result among all the candidates. For each $b$, we solved \eqref{scca-vector-L1} by A-ManPG with $\delta=10^{-4},\gamma = 0.5,t_1=t_2=1$. The A-ManPG was stopped if $\max\{\normfro{D_k^A}^2 , \normfro{D_k^B}^2\}\leq 10^{-8}$ and the regularized SSN was stopped if $\normfro{E(\Lambda_k)}\leq 10^{-5}$ in \eqref{sub_VI}. For AMA+LADMM, we set the stopping criteria of LADMM as $\normtwo{u_j - u_{j-1}}\leq 10^{-3}$ and $\normtwo{v_j - v_{j-1}}\leq 10^{-3}$, where $u_j$ and $v_j$ are iterates in LADMM. We set the stopping criteria of AMA as $\normtwo{u_k - u_{k-1}}\leq 10^{-3}$ and $\normtwo{v_k - v_{k-1}}\leq 10^{-3}$, where $u_k$ and $v_k$ are iterates in AMA.}

We report the numerical results in Table \ref{tab:scca1}, where 'nu' and 'nv' denote the number of nonzeros in $u$ and $v$ after setting their entries whose magnitudes are smaller than $10^{-4}$ to $0$, and $\rho$ denotes the canonical correlation computed from the solution returned by the algorithms. All reported values in Table \ref{tab:scca1} are the medians from $20$ repetitions. From Table \ref{tab:scca1} we see that A-ManPG and AMA+LADMM achieve similar loss function values 'lossu' and 'lossv', but A-ManPG is usually faster than AMA+LADMM, and for some cases, it is even two to three times faster. 
More importantly, AMA+LADMM lacks convergence analysis, but A-ManPG is guaranteed to converge to a stationary point (see the Appendix). Furthermore, AMA+LADMM is very time consuming for the multiple Sparse CCA \eqref{scca-mat-L21}, but A-ManPG is suitable for \eqref{scca-mat-L21} as we show in the next section.

\input{table-scca-1.tex}

\subsection{Sparse CCA: Matrix Case}

In this section, we apply A-ManPG to solve the multiple Sparse CCA \eqref{scca-mat-L21} and compare its performance with CoLaR method proposed by Gao \etal in \cite{gao2017sparse}.
CoLaR is a two-stage method based on convex relaxations. In the first stage, CoLaR solves the following convex program
\be\label{scca:convex_r}
\ba{ll}
\min_F & -\Tr(F^\top X^\top Y)+\tau\normone{F},\\
\st & \normtwo{(X^\top X)^{\half}F(Y^\top Y)^{\half}}_2\leq 1, \normtwo{(X^\top X)^{\half}F(Y^\top Y)^{\half}}_{*} \leq r,
\ea
\ee
where $\|\cdot\|_2$ and $\|\cdot\|_*$ respectively denote the operator norm and nuclear norm, $F$ is the surrogate of $AB^\top$ and the constraint is the convex hull of $\{AB^\top: A\in \St(p,r), B\in \St(q,r)\}$. Here $\St(p,r)$ denotes the Stiefel manifold with matrix size $p\times r$.
Gao \etal \cite{gao2017sparse} suggest to use ADMM to solve \eqref{scca:convex_r}. The main purpose of the first stage is to provide a good initialization for the second stage. Assume solution to \eqref{scca:convex_r} is $\hat{F}$, and $A_0$ and $B_0$ are matrices whose column vectors are respectively the top $r$ left and right singular vectors of $\hat{F}$. A refinement of $A_0$ is adopted in the second stage, in which the following group Lasso problem is solved:
\be\label{group-lasso}
\min_L \Tr(L^\top (X^\top X)L)-2\Tr(L^\top X^\top YB_0) + \tau'\sum_{j=1}^p \normtwo{L_{j\cdot}}.
\ee
A similar strategy for $B_0$ is taken. Suppose the solutions to the group Lasso problems are $A_1$ and $B_1$, the final estimations are normalized as
$A = A_1(A_1^\top X^\top X A_1)^{-\half}$ and $B =B_1(B_1^\top Y^\top Y B_1)^{-\half}$. We found that the efficiency of CoLaR highly relies on the first stage. Since a good initialization is crucial for the nonconvex problem, we also use the solution returned from the first stage \eqref{scca:convex_r} as the initial point for our A-ManPG algorithm.
We follow the same settings of all numerical tests as suggested in \cite{gao2017sparse}. We tested three different ways to generate the covariance matrix $\Sigma_x = \Sigma_y$ with $p=q$.
\begin{itemize}
\item Identity matrices: $\Sigma_x = \Sigma_y = I_p$.
\item Toeplitz matrices: $[\Sigma_x]_{ij} = [\Sigma_y]_{ij} = 0.3^{\abs{i-j}}$.
\item Sparse inverse matrices: $[\Sigma_x]_{ij}=[\Sigma_y]_{ij}= {\sigma^0_{ij}}/{\sqrt{\sigma^0_{ii}\sigma_{jj}^0}}$, where $\Sigma^0=(\sigma^0_{ij})=\Omega^{-1}$ and
    \[\Omega_{ij}=\iota_{(i=j)} + 0.5\times\iota_{\abs{i-j}=1}+0.4\times\iota_{\abs{i-j}=2}.\]
\end{itemize}
In all tests, we chose $r=2$ and generated $\Sigma_{xy}=\Sigma_x U\Lambda V^\top \Sigma_y$, where $\Lambda\in\br^{r\times r}$ is a diagonal matrix with diagonal entries $\Lambda_{11} = 0.9$ and $\Lambda_{22} = 0.8$. The nonzero rows of both $U$ and $V$ are set at the $\{1, 6, 11, 16, 21\}$-th rows. The values at the nonzero coordinates are obtained from normalizing (with respect to $\Sigma_x$ and $\Sigma_y$) random numbers drawn from the uniform distribution on the finite set $\{-2,-1,0,1,2\}$. The two datasets $X\in\br^{n\times p}$ and $Y\in\br^{n\times q}$ are then generated from \eqref{data-normal}. The matrices $X$ and $Y$ are both divided by $\sqrt{n-1}$ such that $X^\top Y$ is the estimated covariance matrix.  The loss between the estimation $A$ and the ground truth $U$ is measured by the subspace distance $\mbox{lossu}=\normfro{P_U - P_{A}}^2$, where $P_U$ denotes the projection matrix onto the column space of $U$. Similarly, the loss for $B$ and $V$ is measured as $\mbox{lossv}=\normfro{P_V - P_{B}}^2$.

The codes of CoLaR were downloaded from the authors' webpage\footnote{http://www-stat.wharton.upenn.edu/$\sim$zongming/research.html}. We used all default settings of their codes. In particular, ADMM is used to solve the first stage problem \eqref{scca:convex_r} and it is terminated when it does not make much progress, or it reaches the maximum iteration number 100. For our A-ManPG, we run only one iteration of ADMM for \eqref{scca:convex_r} and use the returned solution as the initial point of A-ManPG, because we found that this already generates a very good solution for A-ManPG. To be fair, we also compare the same case for CoLaR where only one iteration of ADMM is used for \eqref{scca:convex_r}. The parameter $\tau$ in \eqref{scca:convex_r} is set to $\tau = 0.55\sqrt{\log(p+q)/n}$. {We set $\tau_1 = \tau_2 =\half b\sqrt{\log(p+q)/n}$ in \eqref{scca-mat-L21} and $\tau' =b$ in \eqref{group-lasso} where $b$ was set to $b=\{0.8,1,1.2,1.4,1.6\}$.}
For each $b$, we solved \eqref{scca-mat-L21} by A-ManPG with $\delta=10^{-4},\gamma = 0.5,t_1=t_2=1$. The A-ManPG was stopped if $\max\{\normfro{D_k^A}^2 , \normfro{D_k^B}^2\}\leq 10^{-8}$ and the regularized SSN was stopped if $\normfro{E(\Lambda_k)}\leq 10^{-5}$ in \eqref{sub_VI}. 

We report the numerical results in Tables \ref{tab:scca-init}-\ref{tab:scca-mat-3}, where CPU times are in seconds, 'nA' and 'nB' denote the number of nonzeros of $A$ and $B$ respectively, after truncating the entries whose magnitudes are smaller than $10^{-4}$ to zeros. $\rho_1$ and $\rho_2$ are the two canonical correlations and they should be close to $0.9$ and $0.8$, respectively. All reported values in Tables \ref{tab:scca-init}-\ref{tab:scca-mat-3} are the medians from $20$ repetitions. More specifically, Table \ref{tab:scca-init} reports the results obtained from the first stage where ADMM was used to solve \eqref{scca:convex_r}. 'Init-1' indicates that we only run one iteration of ADMM, and 'Init-100' indicates the case where we run ADMM until it does not make much progress, or the maximum iteration number 100 is reached. From Table \ref{tab:scca-init} we see that solving the first stage problem by running ADMM for 100 iterations indeed improves the two losses significantly.
Tables \ref{tab:scca-mat-1}-\ref{tab:scca-mat-3} report the results for the three different types of covariance matrices. 
A-ManPG-1 and CoLaR-1 are the cases where we only run one iteration of ADMM for the first stage, and CoLaR-100 is the case where the first stage problem \eqref{scca:convex_r} is solved more accurately by ADMM, as discussed above. 
We observed that running more iteration of ADMM in the first stage does not help much for A-ManPG; we thus only report the results of A-ManPG-1.
From Tables \ref{tab:scca-mat-1}-\ref{tab:scca-mat-3} we see that CoLaR-100 gives much better results than CoLaR-1 in terms of the two losses 'lossu' and 'lossv', especially when the sample size is relatively small compared with the matrix sizes. Moreover, we see that A-ManPG-1 outperforms both CoLaR-1 and CoLaR-100 significantly. In particular, A-ManPG-1 generates comparable and very often better solutions than CoLaR-1 and CoLaR-100 in terms of solution sparsity and losses 'lossu' and 'lossv'. Furthermore, A-ManPG-1 is usually faster than CoLaR-1 and much faster than CoLaR-100.


\input{table-scca-mat.tex}

\section{Conclusion}
In this paper, we proposed an efficient algorithm for solving two important and numerically challenging optimization problems arising from statistics: sparse PCA and sparse CCA. These two problems are challenging to solve because they are manifold optimization problems with nonsmooth objectives, a topic that is still underdeveloped in optimization. We proposed an alternating manifold proximal gradient method (A-ManPG) to solve these two problems. Convergence and convergence rate to a stationary point of the proposed algorithm are established. Numerical results on statistical data demonstrate that A-ManPG is comparable to existing algorithms for solving sparse PCA, and is significantly better than existing algorithms for solving sparse CCA.

\appendix

\section{Preliminaries on Manifold Optimization}\label{sec:append-A}

We now introduce some preliminaries on manifold optimization. An important concept in manifold optimization is retraction, which is defined as follows.
\begin{definition}\cite[Definition 4.1.1]{Absil2009}
	A retraction on a differentiable manifold $\mathcal{M}$ is a smooth mapping from the tangent bundle $\T\mathcal{M}$ onto $\mathcal{M}$ satisfying the following two conditions, where $R_X$ denotes the restriction of $R$ onto $\T_X \mathcal{M}$.
	\begin{enumerate}
		\item $R_X(0)=X, \forall X\in\M$, where $0$ denotes the zero element of $\T_X\mathcal{M}$.
		\item For any $X\in\M$, it holds that
		\[\lim_{\T_X\M\ni\xi\rightarrow 0}\frac{\|R_X(\xi)-(X+\xi)\|_F}{\|\xi\|_F} = 0.\]
	\end{enumerate}
\end{definition}

Common retractions on the Stiefel manifold {$\St(p,r) = \{ X:X^\top X =I_r, X\in\R^{p\times r}\}$} include the polar decomposition
\[R_{X}^{\text{polar}}(\xi)=(X+\xi)(I_r+\xi^\top\xi)^{-1/2},\]
the QR decomposition
\[R_{X}^{\text{QR}}(\xi)=\text{qf}(X+\xi),\]
where $\text{qf}(A)$ is the $Q$ factor of the QR factorization of $A$,
and the Cayley transformation
\[R_{X}^{\text{cayley}}(\xi)=(I_p-\frac{1}{2}W(\xi))^{-1}(I_p+\frac{1}{2}W(\xi))X,\]
where $W(\xi)=(I_p-\frac{1}{2}XX^\top)\xi X^\top-X\xi^\top(I_p-\frac{1}{2}XX^\top)$. 
In our numerical tests, we chose the polar decomposition for retraction.

\subsection{Preliminaries of Generalized Stiefel manifold}
We denote the generalized Stiefel manifold as $\M=\GSt(p,r)=\{U\in\R^{p\times r}: U^\top M U=I_r\}$, where $M\in\br^{p\times p}$ is positive definite. The tangent space of $\GSt(p,r)$ at $U$ is given by $\T_U\M =\{\delta: \delta^\top MU+U^\top M\delta=0 \}$.

The generalized polar decompostion of a tangent vector $Y\in \T_U\M$ can be computed as follows:
\[ R_Y^{\text{polar}}=\bar{U}(Q\Lambda^{-1/2}Q^\top)\bar{V}^\top,\]
where $\bar{U}\Sigma \bar{V}^\top=Y$ is the truncated SVD of $Y$, and $Q,\Lambda$ are obtained from  the eigenvalue decomposition $Q\Lambda Q^\top=\bar{U}^\top M\bar{U}$.

\subsection{Optimality Condition of Manifold Optimization}
\begin{definition}(Generalized Clarke subdifferential \cite{hosseini2011generalized})
	For a locally Lipschitz function $F$ on $\M$, the Riemannian generalized directional derivative of $F$ at $X\in\M$ in direction $V$ is defined by
	\be\label{r_dir_derivative}
	F^{\circ}(X,V) =\limsup\limits_{Y\rightarrow X,t\downarrow 0}\frac{F\circ \phi^{-1}(\phi(Y)+tD\phi(X )[V])-f\circ \phi^{-1}(\phi(Y))}{t},
	\ee
	where $(\phi,U)$ is a coordinate chart at $X$ and $D\phi(X)$ denotes the Jacobian  of $\phi(X)$.
	The generalized gradient or the Clarke subdifferential of $F$ at $X\in\M$, denoted by $\hat{\partial} F(X)$, is given by
	\be\label{r_clarke_sub}
	\hat{\partial} F(X)=\{\xi\in \T_X\M :\inp{\xi}{V}\leq F^{\circ}(X,V), \ \forall V\in \T_X\M \}.
	\ee
\end{definition}
\begin{definition}(\cite{Yang-manifold-optimality-2014})
	A function $f$ is said to be regular at $X\in\M$ along $\T_X\M$ if
	\begin{itemize}
		\item for all $V\in \T_X\M$, $f'(X;V)=\lim_{t\downarrow 0} \frac{f(X+tV)-f(X)}{t}$ exists, and
		\item for all $V\in \T_X\M$, $f'(X;V) = f^\circ (X;V)$.
	\end{itemize}
\end{definition}
For smooth function $f$, we know that $\grad f(X)= \Proj_{\T_X\M} \nabla f(X)$ since the metric on the manifold is the Euclidean Frobenius metric. Here $\grad f$ denotes the Riemannian gradient of $f$, and $\Proj_{\T_X\M}$ denotes the projection onto ${\T_X\M}$. According to Theorem 5.1 in \cite{Yang-manifold-optimality-2014}, for a regular function $F$, we have $\hat{\partial}F(X)=\Proj_{\T_X\M}(\partial F(X)).$ Moreover, let $X=(A,B)$, the function $F(X)=H(X)+f(A)+g(B)$ in problem \eqref{general_prob} is regular according to Lemma 5.1 in \cite{Yang-manifold-optimality-2014}. Therefore, we have $\hat{\partial}F(X)=\grad F(A,B)+ \Proj_{\T_A\M_1}(\partial f(A))+ \Proj_{\T_B\M_2}(\partial g(B))$. By Theorem 4.1 in \cite{Yang-manifold-optimality-2014}, the first-order optimality condition of problem \eqref{general_prob} is given by
\be\label{opt-cond}0\in \grad H(A,B)+ \Proj_{\T_A\M_1}(\partial f(A))+ \Proj_{\T_B\M_2}(\partial g(B)).\ee

\begin{definition}\label{stationary_point}
	A point $X\in\M$ is called a stationary point of problem \eqref{general_prob} if it satisfies the first-order optimality condition \eqref{opt-cond}. 
\end{definition}

\section{Semi-smoothness of Proximal Mapping}

\begin{definition}
	Let $F:\Omega\rightarrow \R^q$ be locally Lipschitz continuous at $X\in \Omega\subset \R^p$. The $B$-subdifferential of $F$ at $X$ is defined by
	\[ \partial_B F(X):= \{\lim_{k\rightarrow \infty} F'(X_k)| X^k\in D_F, X_k \rightarrow X \}, \]
	where $D_F$ be the set of differentiable points of $F$ in $\Omega$. The set $\partial F(X) = \conv(\partial_B F(X))$ is called Clarke's generalized Jacobian, where $\conv$ denotes the convex hull.
\end{definition}
{
Note that if $q = 1$ and $F$ is convex, then the definition is the same as that of standard convex subdifferential. So, we use the notation $\partial$ for the general purpose.}
\begin{definition}\cite{Mifflin-1977,Qi-Sun-1993}
		Let $F:\Omega\rightarrow \R^q$ be locally Lipschitz continuous at $X\in \Omega\subset \R^p$. We say that $F$ is semi-smooth at $X\in \Omega$ if $F$ is directionally differentiable at $X$ and for any $J\in \partial F(X+\Delta X)$ with $\Delta X\rightarrow 0$,
		\[F(X+\Delta X) - F(X) - J\Delta X = o(\normtwo{\Delta X}). \]
		We say  $F$ is strongly semi-smooth if  $F$ is semi-smooth at $X$ and
		\[F(X+\Delta X) - F(X) - J\Delta X = O(\normtwo{\Delta X}^2). \]	
		We say that $F$ is a semi-smooth function on $\Omega$ if it is semi-smooth everywhere in $\Omega$.
\end{definition}

\section{Global Convergence of A-ManPG (Algorithm \ref{alg:amgdpgd})}\label{sec:append-B}

To show the global convergence of A-ManPG, we need the following assumptions for problem \eqref{general_prob}, which are commonly used in first-order methods.

\begin{assumption}
    \begin{itemize}
 	\item[(1)] $f$ and $g$ are convex and Lipschitz continuous with Lipschitz constants $L_f$ and $L_g$, respectively.
 	\item[(2)] $\nabla_A H(A,B)$ is Lipschitz continuous with respect to $A$ when fixing $B$, and the Lipschitz constant is $L_A$. Similarly, $\nabla_B H(A,B)$ is Lipschitz continuous with respect to $B$ when fixing $A$, and the Lipschitz constant is $L_B$.
    \item[(3)] {$F$ is lower bounded by a constant $F^*$. }
    \end{itemize}
\end{assumption}
Note here the Lipschitz continuity and convexity are all defined in the Euclidean space.


We prove that the sequence generated by A-ManPG converges to stationary point of \eqref{general_prob} in this section. We need the following two properties of retraction whose proofs can be found in \cite{Boumal2016}.
 \begin{lemma}\label{retraction:property}
 	Let ${\M} $ be a compact embedded submanifold in Euclidean space.
 	For all $X\in {\M}$ and $\xi \in \T_X\M$, there exist constants $M_1>0$ and $M_2>0$ such that the following two inequalities hold:
 	\be\label{first-bounded}
 	\normfro{R_X(\xi)-X}\leq M_1\normfro{\xi}, \forall X\in {\M}, \xi \in \T_X\M,
 	\ee
 	\be\label{Second-bounded}
 	\normfro{R_X(\xi)-(X+\xi)}\leq M_2\normfro{\xi}^2, \forall X\in {\M}, \xi \in \T_X\M.
 	\ee
 \end{lemma}
Note that the (generalized) Stiefel manifold is compact, so the two inequalities hold naturally.
{
\begin{definition}
	A function $f(X)$ is $\alpha$-strongly convex in $\R^p$ if
	\[ f(Y)\geq f(X) +\inp{\partial f(X)}{Y-X}+\frac{\alpha}{2}\normtwo{Y-X}^2\]
	holds for $\forall X,Y\in\R^p$.
\end{definition}
}
The following lemma shows that $D^A_k$ and $D^B_k$ obtained from \eqref{sub_amanpg} are descent directions in the tangent space.
\begin{lemma}\label{tangent-des}
    The following inequalities hold for any $\alpha \in[0,1]$ if $t_1\leq 1/L_A,t_2\leq 1/ L_B$:
	\be\label{d1}
	H(A_k+\alpha D^A_k,B_k) + f(A_k+\alpha D_k^A) \leq H(A_k,B_k)+f(A_k)-\frac{\alpha }{2t_1}\normtwo{D^A_k}_F^2,
	\ee
	\be\label{d2}
	H(A_{k+1},B_k+\alpha D^B_k) + g(B_k+\alpha D^B_k) \leq H(A_{k+1},B_k)+g(B_k)-\frac{\alpha }{2t_2}\normtwo{D^B_k}_F^2.
	\ee
\end{lemma}

\begin{proof}
	For simplicity, we only prove inequality \eqref{d1}. The proof of \eqref{d2} is similar.
	{Since the objective function
	$G(D):=\inp{\nabla_A H(A_k,B_k)}{ D}+\frac{1}{2t_1}\normtwo{ D}_F^2+f(A_k+ D)$ is $1/t_1$-strongly convex, we have
    	\be\label{e1}
    G(\hat{D})\geq G(D) + \inp{\partial G(D)}{\hat{D}-D} + \frac{1}{2t_1}\normfro{\hat{D}-D}^2, \ \forall D, \hat{D}.
    \ee
Specifically, if $D,\hat{D}$ are feasible, $\ie$, $D,\hat{D}\in \T_{A_k}\M_1$, we have $\inp{\partial G(D)}{\hat{D}-D} = \inp{\Proj_{\T_{A_k}\M_1}\partial G(D)}{\hat{D}-D}$.
From the optimality condition of \eqref{sub_amanpg}, we have $0\in \Proj_{\T_{A_k}\M_1}\partial G(D^A_k)$. Letting $D = D_k^A$, $\hat{D} = \alpha D_k^A$, $\alpha\in[0,1]$ in \eqref{e1} yields
 \[G(\alpha D_k^A)\geq G(D^A_k) +  \frac{(1-\alpha)^2}{2t_1}\normfro{D^A_k}^2, \] which implies }
	\be\label{e2}
	\bad
	&\inp{\nabla_A H(A_k,B_k)}{\alpha D_k^A}+\frac{1}{2t_1}\normtwo{\alpha D_k^A}_F^2+f(A_k+\alpha D_k^A)\\
	\geq& \inp{\nabla_A H(A_k,B_k)}{D_k}+\frac{1}{2t_1}\normtwo{D_k^A}_F^2+f(A_k+D_k^A) + \frac{(1-\alpha)^2}{2t_1}\normfro{D^A_k}^2,
	\ead\ee

	Combining with the convexity of $f$, \eqref{e2} yields
	\be\label{d11}
	\bad
	(1-\alpha)\inp{\nabla_A H(A_k,B_k)}{D_k^A}+\frac{1-\alpha}{t_1}\normtwo{D^A_k}_F^2+(1-\alpha)(f(A_k+D^A_k)-f(A_k))\leq 0.
	\ead\ee
	Combining the convexity of $f$ and the Lipschitz continuity  of $\nabla_A H(A,B_k)$, we have
	\begin{align*}
	&H(A_k+\alpha D_k^A,B_k)-H(A_k,B_k) + f(A_k+\alpha D^A_k)-f(A_k) \\
	\leq &\alpha\inp{\nabla_A H(A_k,B_k)}{D^A_k} +\frac{\alpha^2}{2t_1}\normtwo{D_k^A}_F^2 + \alpha(f(A_k+D_k)-f(A_k)) \\
	\leq  &-\frac{\alpha}{2t_1}\normtwo{D_k^A}_F^2,
	\end{align*}
	where the last inequality holds from $\alpha\in[0,1]$ and \eqref{d11}.
\end{proof}

The following lemma shows that if one cannot make any progress by solving \eqref{sub_amanpg}, i.e., $D^A_k=0,D^B_k=0$, then a stationary point is found.
\begin{lemma}\label{first_order_opt}
	If $D^A_k=0$ and $D^B_k=0$, then $(A_k,B_k)$ is a stationary point of problem \eqref{general_prob}.
\end{lemma}

\begin{proof}
	By Theorem 4.1 in \cite{Yang-manifold-optimality-2014}, the optimality conditions for the $A$-subproblem in \eqref{sub_amanpg} are
	given by
	\[ 0\in D^A_k/t_1 + \grad_A H(A_k,B_k)+\Proj_{\T_{A_k}\M_1} \partial f(A_k+D), \quad \text{and}\ D^A_k\in \T_{A_k} \M_1. \]
	If $D^A_k=0$, it follows that
	\be\label{opt1} 0\in\grad_A H(A_k,B_k)+\Proj_{\T_{A_k}\M_1} \partial f(A_k). \ee
	Similarly, if $D^B_k=0$, we obtain
	\be\label{opt2} 0\in\grad_B H(A_k,B_k)+\Proj_{\T_{B_k}\M_2} \partial g(B_k). \ee
	Combining \eqref{opt1} and \eqref{opt2} yields the first-order optimality condition of problem \eqref{general_prob} since $(A_k,B_k)\in (\M_1,\M_2)$.
\end{proof}

\begin{lemma}\label{sufficient_des}
	There exist constants $\bar{\alpha}_1, \bar{\alpha}_2>0$ and $\bar{\beta}_1, \bar{\beta}_2>0$ such that for any $0<\alpha_1 \leq \min\{1,\bar{\alpha}_1\}$, $0<\alpha_2 \leq \min\{1,\bar{\alpha}_2\}$, the sequence $\{(A_k,B_k)\}$ generated by Algorithm \ref{alg:amgdpgd} satisfies the following inequalities:
	\be\label{suff_des_ineq1}
	F(A_{k+1},B_{k}) -F(A_k,B_k) \leq -\bar{\beta}_1\normtwo{D^A_k}_F^2,
	\ee
		\be\label{suff_des_ineq2}
	F(A_{k+1},B_{k+1}) -F(A_{k+1},B_k) \leq -\bar{\beta}_2\normtwo{D^B_k}_F^2.
	\ee
\end{lemma}

\begin{proof}
	We prove \eqref{suff_des_ineq1} by induction, and the proof of \eqref{suff_des_ineq2} is similar and thus omitted. Define ${A}^+_k= A_k+\alpha_1 D^A_k$.
	For $k=0$, 
	by Lemma \ref{retraction:property}, \eqref{first-bounded} and \eqref{Second-bounded} hold for $X=A_0$.
	Note that $A_{k+1}=R_{A_k}(\alpha_1 D^A_k)$. From the Lipschitz continuity of $\nabla H(A,B_k)$, we have
	\be\label{ineq1}
	\bad
	& H(A_{k+1},B_k)-H(A_k,B_k) \\
    \leq & \inp{\nabla_A H(A_k,B_k)}{A_{k+1}-A_k}+\frac{L_A}{2}\normfro{A_{k+1}-A_k}^2\\
    = & \inp{\nabla_A H(A_k,B_k)}{A_{k+1}- {A}_k^+ + {A}_k^+ -A_k}+\frac{L_A}{2}\normfro{A_{k+1}-A_k}^2\\
	\leq & M_2\normfro{\nabla_A H(A_k,B_k)}\normfro{\alpha_1 D^A_k}^2 + \alpha_1\inp{\nabla_A H(A_k,B_k)}{D^A_k} + \frac{L_AM_1}{2}\normfro{\alpha_1 D^A_k}^2,
	\ead
	\ee
	where the last inequality is due to \eqref{first-bounded} and \eqref{Second-bounded}. Since $\nabla_A H(A,B_k)$ is continuous on the compact set $\M_1$, there exists a constant $G>0$ such that $\normfro{\nabla_A H(A,B_k)}\leq G$ for all $A\in\M_1$. It then follows from \eqref{ineq1} that
	\be\label{smooth_ine}
		H(A_{k+1},B_k)-	H(A_{k},B_k) \leq c_0\alpha_1^2\normfro{D^A_k}^2 + \alpha_1\inp{\nabla_A H(A_k,B_k)}{D^A_k},
	\ee
	where $c_0=M_2G +L_AM_1/2$. From \eqref{smooth_ine} we can show the following inequalities: 
	\be\label{suff-dec}
	\bad
	& F(A_{k+1},B_k)-F(A_k,B_k) \\
   \overset{\eqref{smooth_ine}}{\leq } & \alpha_1\inp{\nabla_A H(A_k,B_k)}{D^A_k} +c_0 \alpha_1^2 \normtwo{D^A_k}_F^2+
	f(A_{k+1})-f(A_k^+)+ f(A_k^+)- f(A_k)  \\
	\leq \ &\alpha_1\inp{\nabla_A H(A_k,B_k)}{D^A_k} +c_0 \alpha_1^2 \normtwo{D^A_k}_F^2+
	L_f\normfro{A_{k+1}-{A}_k^+}+\alpha_1( f(A_k+D^A_k)- f(A_k))\\
	\overset{\eqref{Second-bounded}}{\leq } & (c_0 \alpha_1^2+L_f M_2\alpha_1^2) \normtwo{D^A_k}_F^2+ \alpha_1\left[ \inp{\nabla_A H(A_k,B_k)}{D^A_k} +f(A_k+D^A_k)-f(A_k)\right]\\
	\overset{\eqref{d11}}{\leq } & \left[(c_0+L_fM_2)\alpha_1^2-\alpha_1/t_1\right]\normtwo{D^A_k}_F^2,
	\ead
	\ee
	where the second inequality follows from the Lipschitz continuity of $f(A)$.
	Define function $\beta(\alpha_1)=-(c_0+L_f M_2)\alpha_1^2+\alpha_1/t_1$, $\bar{\alpha}_1=\frac{1}{2(c_0+L_f M_2)t_1}$. It is easy to see from \eqref{suff-dec} that
	\[
	F(A_{k+1},B_k)-F(A_k,B_k)\leq -\bar{\beta}_1 \normtwo{D_A^k}_F^2, \quad \text{if}\  0<\alpha_1 \leq \min\{1,\bar{\alpha}_1\},
	\]
	where
	\[
	\bar{\beta}_1=\left\{\begin{matrix}
	\beta(\alpha_1)&\text{if}\ \bar{\alpha}_1\leq 1,\\
	\beta(1) & \text{if}\ \bar{\alpha}_1 > 1 .
	\end{matrix}\right.
	\]
	Thus, \eqref{suff_des_ineq1} holds for $k=0$. 
	 Suppose that \eqref{suff_des_ineq1} holds for $k\geq 1$, with the same argument, it follows that \eqref{suff_des_ineq1} holds for $k+1$ and $A_{k+1}\in\M_1$.
\end{proof}

Now we are ready to give the proof of Theorem \ref{thm:complexity}.

\begin{proof}
	By Lemma \ref{sufficient_des} and the lower boundedness of $F(A,B)$, we have
	$$\lim_{k\rightarrow \infty} (\bar{\beta}_1 \normfro {D^A_k}^2+ \bar{\beta}_2 \normfro {D^B_k}^2)=0.$$
	Combining with Lemma \ref{first_order_opt}, it follows that any limit point of $\{(A_k,B_k)\}$ is a stationary point of \eqref{general_prob}. Moreover, since $\M_1$ and $\M_2$ are
	compact, there exists at least one limit point of the sequence  $\{(A_k,B_k)\}$.
	
	Furthermore, suppose that Algorithm \ref{alg:amgdpgd} does not terminate after $K$ iterations, $\ie$, $\bar{\beta}_1 \normfro {D^A_k}^2+ \bar{\beta}_2\normfro {D^B_k}^2  > \epsilon^2$ for all $k=0,1,\ldots,K-1$. In this case, we have
	$F(A_0,B_0)-F^*\geq F(A_0,B_0)-F(A_K,B_K)\geq (\bar{\beta}_1 + \bar{\beta}_2)\sum_{k=0}^{K-1}(\normfro{D_k^A}^2+\normfro{D_k^B}^2)>(\bar{\beta}_1 + \bar{\beta}_2)K\epsilon^2 $.
	Therefore, Algorithm \ref{alg:amgdpgd} finds an $\epsilon$-stationary point, after $K\geq (F(A_0,B_0)-F^*)/((\bar{\beta}_1 + \bar{\beta}_2)\epsilon^2)$ iterations.
\end{proof}

%
%
%
%

\bibliography{manifold,NSF-nonsmooth-manifold,manifold1}
\bibliographystyle{plain}
\end{document}

%% file: table-spca-1.tex
\begin{table}[htbp]\small
	\centering
	\caption{Comparison of the algorithms for solving \eqref{Zou-spca} with $n < p$.} 
	\begin{tabular}{c|rrrr|rrrr}
		\hline
		& \multicolumn{4}{c|}{ $\mu =1$} & \multicolumn{4}{c}{ $\mu =10$} \bigstrut\\
		\hline
		& $F(A,B)$ & sp    & CPU   & iter  & $O(A,B)$ & sp    & CPU   & \multicolumn{1}{c}{iter}  \bigstrut\\
		\hline
		\multicolumn{9}{c}{$(n,p)=(100, 1000), \mu_{1,j}=0.1, j=1,\ldots,r$} \bigstrut\\
		\hline
		AMA & -4.90778e+1 & 59.7  & 11.48  & 648   & -2.69714e+1 & 25.2  & 4.06  & 409   \bigstrut[t]\\
		A-ManPG & -4.90771e+1 & 59.4  & 0.36  & 1172  & -2.69716e+1 & 25.4  & 0.12  & 375   \\
		PALM  & -4.90769e+1 & 59.4  & 0.39  & 1394  & -2.69711e+1 & 25.3  & 0.15  & 518   \\
		VP & -4.90770e+1 & 59.4  & 0.37  & 1335  & -2.69712e+1 & 25.2  & 0.13  & 453  \bigstrut[b]\\
		\hline
		\multicolumn{9}{c}{$(n,p)=(100, 1000), \mu_{1,j}=0.2, j=1,\ldots,r$} \bigstrut\\
		\hline
		AMA & -4.16070e+1 & 76.5  & 7.28  & 433   & -2.16374e+1 & 42.6  & 2.78  & 259    \bigstrut[t]\\
		A-ManPG & -4.16057e+1 & 76.4  & 0.22  & 712   & -2.16371e+1 & 42.7  & 0.09  & 265  \\
		PALM  & -4.16055e+1 & 76.4  & 0.23  & 855   & -2.16371e+1 & 42.6  & 0.12  & 343  \\
		VP & -4.16056e+1 & 76.4  & 0.23  & 825   & -2.16372e+1 & 42.6  & 0.10  & 301 \bigstrut[b]\\
		\hline
		\multicolumn{9}{c}{$(n,p)=(500, 1000), \mu_{1,j}=0.1, j=1,\ldots,r$} \bigstrut\\
		\hline
		AMA & -1.47159e+1 & 66.4  & 8.67  & 543   & -5.31315e+0 & 42.5  & 3.67  & 293      \bigstrut[t]\\
		A-ManPG & -1.47158e+1 & 66.3  & 0.39  & 798   & -5.31838e+0 & 42.3  & 0.23  & 427   \\
		PALM  & -1.47155e+1 & 66.3  & 0.46  & 1044  & -5.31250e+0 & 42.4  & 0.24  & 495   \\
		VP & -1.47157e+1 & 66.4  & 0.38  & 883   & -5.31310e+0 & 42.6  & 0.15  & 319  \bigstrut[b]\\
		\hline
		\multicolumn{9}{c}{$(n,p)=(500, 1000), \mu_{1,j}=0.2, j=1,\ldots,r$} \bigstrut\\
		\hline
		AMA & -1.00053e+1 & 87.3  & 7.06  & 464   & -3.19608e+0 & 68.3  & 4.95  & 386      \bigstrut[t]\\
		A-ManPG & -9.98687e+0 & 86.9  & 0.24  & 486   & -3.18822e+0 & 68.3  & 0.13  & 183   \\
		PALM  & -9.98680e+0 & 87.1  & 0.24  & 533   & -3.18791e+0 & 68.0  & 0.22  & 439   \\
		VP & -9.98688e+0 & 87.2  & 0.21  & 445   & -3.19602e+0 & 68.2  & 0.22  & 445  \bigstrut[b]\\
		\hline
		\multicolumn{9}{c}{$(n,p)=(500, 5000), \mu_{1,j}=0.1, j=1,\ldots,r$} \bigstrut\\
		\hline
		AMA & -5.56171e+1 & 75.8  & 728.29  & 1000  & -3.18762e+1 & 40.5  & 452.19  & 1407    \bigstrut[t]\\
		A-ManPG & -5.56134e+1 & 75.7  & 8.90  & 1982  & -3.18753e+1 & 40.5  & 4.73  & 1021    \\
		PALM  & -5.56131e+1 & 75.8  & 10.77  & 2210  & -3.18550e+1 & 40.3  & 4.94  & 1023    \\
		VP & -5.56132e+1 & 75.7  & 10.87  & 2147  & -3.18759e+1 & 40.5  & 7.17  & 1643   \bigstrut[b]\\
		\hline
		\multicolumn{9}{c}{$(n,p)=(500, 5000), \mu_{1,j}=0.2, j=1,\ldots,r$} \bigstrut\\
		\hline
		AMA & -4.25661e+1 & 89.3  & 733.36  & 1000  & -2.18082e+1 & 63.7  & 171.90  & 545   \bigstrut[t]\\
		A-ManPG & -4.25408e+1 & 89.0  & 9.05  & 2017  & -2.18085e+1 & 63.6  & 3.28  & 700    \\
		PALM  & -4.25111e+1 & 89.1  & 7.59  & 1713  & -2.18079e+1 & 63.6  & 4.01  & 870    \\
		VP & -4.25115e+1 & 89.1  & 7.28  & 1682  & -2.18080e+1 & 63.6  & 3.57  & 773  \bigstrut[b]\\
		\hline
		\multicolumn{9}{c}{$(n,p)=(1000, 5000), \mu_{1,j}=0.1, j=1,\ldots,r$} \bigstrut\\
		\hline
		AMA & -2.89684e+1 & 79.6  & 306.42  & 437   & -1.34357e+1 & 50.9  & 204.01  & 534     \bigstrut[t]\\
		A-ManPG & -2.89676e+1 & 79.7  & 9.61  & 959   & -1.34355e+1 & 50.9  & 5.90  & 535     \\
		PALM  & -2.89675e+1 & 79.6  & 10.24  & 1031  & -1.34352e+1 & 50.8  & 8.15  & 794     \\
		VP & -2.89676e+1 & 79.6  & 9.64  & 975   & -1.34355e+1 & 50.8  & 6.74  & 644     \bigstrut[b]\\
		\hline
		\multicolumn{9}{c}{$(n,p)=(1000, 5000), \mu_{1,j}=0.2, j=1,\ldots,r$} \bigstrut\\
		\hline
		AMA & -1.94321e+1 & 93.9  & 398.16  & 666   & -7.41353e+0 & 77.1  & 317.83  & 841    \bigstrut[t]\\
		A-ManPG & -1.94308e+1 & 93.9  & 23.33  & 2346  & -7.41377e+0 & 77.4  & 10.38  & 1033  \\
		PALM  & -1.94306e+1 & 93.9  & 21.54  & 2104  & -7.41316e+0 & 77.1  & 16.54  & 1565   \\
		VP & -1.94306e+1 & 93.9  & 19.27  & 1989  & -7.41354e+0 & 77.1  & 11.33  & 1138  \bigstrut[b]\\
		\hline
	\end{tabular}%
	\label{tab:spca_rand_1}
\end{table}%

%% file: table-spca-2.tex
\begin{table}[htbp]\small
	\centering
	\caption{Comparison of the algorithms for \eqref{Zou-spca} with $n > p$ and $\mu=10^{-6}$. }
	\begin{tabular}{crrrr}
		\hline
		& $F(A,B)$ & sp    & CPU   & iter   \bigstrut\\
		\hline
		\multicolumn{5}{c}{$(n,p)=(5000, 500), \mu_{1,j}=0.01, j=1,\ldots,r$} \bigstrut\\
		\hline
		AMA & -8.54858e+0 & 31.9  & 10.62  & 700    \bigstrut[t]\\
		A-ManPG & -8.54859e+0 & 31.8  & 0.14  & 541  \\
		PALM  & -8.54739e+0 & 31.4  & 0.23  & 1077  \\
		VP & -8.54841e+0 & 31.7  & 0.17  & 757  \bigstrut[b]\\
		\hline
		\multicolumn{5}{c}{$(n,p)=(5000, 500), \mu_{1,j}=0.05, j=1,\ldots,r$} \bigstrut\\
		\hline
		AMA & -6.45735e+0 & 90.1  & 6.29  & 428      \bigstrut[t]\\
		A-ManPG & -6.45546e+0 & 89.8  & 0.11  & 402   \\
		PALM  & -6.45532e+0 & 89.7  & 0.14  & 689   \\
		VP & -6.45698e+0 & 90.0  & 0.12  & 571   \bigstrut[b]\\
		\hline
		\multicolumn{5}{c}{$(n,p)=(5000, 2000), \mu_{1,j}=0.01, j=1,\ldots,r$} \bigstrut\\
		\hline
		AMA & -1.29155e+1 & 37.8  & 133.04  & 539     \bigstrut[t]\\
		A-ManPG & -1.29151e+1 & 37.6  & 2.44  & 493    \\
		PALM  & -1.29142e+1 & 37.6  & 3.54  & 768    \\
		VP & -1.29150e+1 & 37.5  & 2.93  & 640   \bigstrut[b]\\
		\hline
		\multicolumn{5}{c}{$(n,p)=(5000, 2000), \mu_{1,j}=0.05, j=1,\ldots,r$} \bigstrut\\
		\hline
		AMA & -8.83497e+0 & 89.4  & 88.04  & 425      \bigstrut[t]\\
		A-ManPG & -8.83440e+0 & 89.4  & 3.87  & 842    \\
		PALM  & -8.83437e+0 & 89.3  & 4.36  & 977    \\
		VP & -8.83452e+0 & 89.3  & 3.48  & 773   \bigstrut[b]\\
		\hline
		\multicolumn{5}{c}{$(n,p)=(8000, 1000), \mu_{1,j}=0.01, j=1,\ldots,r$} \bigstrut\\
		\hline
		AMA & -9.04522e+0 & 37.9  & 23.53  & 325      \bigstrut[t]\\
		A-ManPG & -9.04477e+0 & 37.6  & 0.22  & 276    \\
		PALM  & -9.04471e+0 & 37.7  & 0.31  & 507    \\
		VP & -9.04511e+0 & 37.8  & 0.24  & 360   \bigstrut[b]\\
		\hline
		\multicolumn{5}{c}{$(n,p)=(8000, 1000), \mu_{1,j}=0.05, j=1,\ldots,r$} \bigstrut\\
		\hline
		AMA & -6.59097e+0 & 95.6  & 44.30  & 636       \bigstrut[t]\\
		A-ManPG & -6.58996e+0 & 95.7  & 0.78  & 897     \\
		PALM  & -6.58995e+0 & 95.7  & 0.82  & 1486    \\
		VP & -6.60764e+0 & 95.9  & 1.22  & 1907 \bigstrut[b]\\
		\hline
		\multicolumn{5}{c}{$(n,p)=(8000, 2000), \mu_{1,j}=0.01, j=1,\ldots,r$} \bigstrut\\
		\hline
		AMA & -1.07975e+1 & 40.3  & 116.65  & 388     \bigstrut[t]\\
		A-ManPG & -1.07975e+1 & 40.2  & 2.19  & 363    \\
		PALM  & -1.07966e+1 & 40.2  & 3.00  & 550    \\
		VP & -1.07972e+1 & 40.4  & 2.70  & 437    \bigstrut[b]\\
		\hline
		\multicolumn{5}{c}{$(n,p)=(8000, 2000), \mu_{1,j}=0.05, j=1,\ldots,r$} \bigstrut\\
		\hline
		AMA & -7.32162e+0 & 95.3  & 167.36  & 597    \bigstrut[t]\\
		A-ManPG & -7.31837e+0 & 95.0  & 3.34  & 578   \\
		PALM  & -7.30781e+0 & 95.0  & 3.97  & 780    \\
		VP & -7.30822e+0 & 95.0  & 3.29  & 606   \bigstrut[b]\\
		\hline
	\end{tabular}%
	\label{tab:spca_rand_2}
\end{table}%

%% file: table-scca-1.tex
\begin{table}[htbp]\small
	\centering
 \caption{Comparison of A-ManPG and AMA+LADMM  \cite{suo2017sparse} for solving single sparse CCA \eqref{scca-vector-L1}.}
	\begin{tabular}{c|cccccc||cccccc}
		\hline
		\hline
		\multicolumn{1}{c}{} & \multicolumn{6}{c}{ManPG}                     & \multicolumn{6}{c}{AMA+LADMM} \bigstrut\\
		\hline
		\hline
		$(n,p,q)$ & cpu   & lossu & lossv & $\rho$ & nu    & nv    & cpu   & lossu & lossv & $\rho$ & nu    & nv \bigstrut\\
		\hline
		\multicolumn{13}{c}{Identity matrix} \bigstrut\\
		\hline
		500,800,800 & 0.265  & 3.955e-3 & 4.635e-3 & 0.900  & 4     & 4.5   & 0.737  & 3.955e-3 & 4.639e-3 & 0.900  & 4     & 4.5 \bigstrut[t]\\
		1000,800,800 & 0.395  & 2.477e-3 & 2.350e-3 & 0.899  & 4     & 4.5   & 1.240  & 2.470e-3 & 2.347e-3 & 0.899  & 4     & 4.5 \\
		500,1600,1600 & 0.990  & 6.071e-3 & 4.247e-3 & 0.898  & 5     & 4.5   & 2.475  & 6.050e-3 & 4.240e-3 & 0.898  & 5     & 4.5 \\
		1000,1600,1600 & 1.244  & 1.351e-3 & 2.081e-3 & 0.900  & 5     & 5     & 3.880  & 1.350e-3 & 2.078e-3 & 0.900  & 5     & 5 \bigstrut[b]\\
		\hline
		\multicolumn{13}{c}{Toeplitz matrix} \bigstrut\\
		\hline
		500,800,800 & 0.279  & 3.569e-3 & 5.570e-3 & 0.902  & 7     & 5.5   & 0.821  & 3.567e-3 & 5.570e-3 & 0.902  & 7     & 5.5 \bigstrut[t]\\
		1000,800,800 & 0.395  & 2.152e-3 & 2.165e-3 & 0.902  & 5     & 5     & 1.337  & 2.151e-3 & 2.159e-3 & 0.902  & 5     & 5 \\
		500,1600,1600 & 0.955  & 5.802e-3 & 4.758e-3 & 0.896  & 4     & 4.5   & 2.600  & 5.800e-3 & 4.751e-3 & 0.896  & 4     & 4.5 \\
		1000,1600,1600 & 1.172  & 1.913e-3 & 1.602e-3 & 0.901  & 5     & 5.5   & 3.644  & 1.913e-3 & 1.604e-3 & 0.901  & 5     & 5.5 \bigstrut[b]\\
		\hline
		\multicolumn{13}{c}{Sparse inverse matrix} \bigstrut\\
		\hline
		500,800,800 & 0.527  & 7.749e-3 & 1.248e-2 & 0.896  & 7     & 6.5   & 0.815  & 7.509e-3 & 1.209e-2 & 0.896  & 6.5   & 7 \bigstrut[t]\\
		1000,800,800 & 0.618  & 5.920e-3 & 4.631e-3 & 0.898  & 5     & 5     & 1.630  & 5.843e-3 & 4.624e-3 & 0.898  & 5     & 5 \\
		500,1600,1,600 & 1.589  & 9.624e-3 & 1.052e-2 & 0.889  & 5     & 5     & 2.822  & 1.010e-2 & 1.031e-2 & 0.889  & 5     & 5 \\
		1000,1600,1600 & 1.951  & 2.799e-3 & 3.812e-3 & 0.900  & 6.5   & 6     & 4.583  & 2.941e-3 & 3.807e-3 & 0.900  & 6.5   & 6 \bigstrut[b]\\
		\hline
	\end{tabular}%
	\label{tab:scca1}%
\end{table}%

%% file: table-scca-mat.tex
\begin{table}[htbp]
	\centering
	\caption{Losses returned from the first stage problem \eqref{scca:convex_r}.}
	\begin{tabular}{ccc|cc}
		\hline
		& \multicolumn{2}{c|}{Init-1} & \multicolumn{2}{c}{Init-100} \bigstrut\\
		\hline
		$(n,p,q)$ & lossu & lossv & lossu & lossv \bigstrut\\
		\hline
		\multicolumn{5}{c}{Identity matrix} \bigstrut\\
		\hline
		200,300,300 & 0.304  & 0.374  & 0.107  & 0.124  \bigstrut[t]\\
		500,300,300 & 0.114  & 0.103  & 0.050  & 0.037  \\
		200,600,600 & 0.394  & 0.393  & 0.146  & 0.116  \\
		500,600,600 & 0.137  & 0.139  & 0.048  & 0.035  \bigstrut[b]\\
		\hline
		\multicolumn{5}{c}{Toeplitz matrix} \bigstrut\\
		\hline
		200,300,300 & 0.318  & 0.375  & 0.120  & 0.107  \bigstrut[t]\\
		500,300,300 & 0.126  & 0.090  & 0.038  & 0.028  \\
		200,600,600 & 0.427  & 0.401  & 0.103  & 0.110  \\
		500,600,600 & 0.101  & 0.133  & 0.028  & 0.039  \bigstrut[b]\\
		\hline
		\multicolumn{5}{c}{Sparse inverse matrix} \bigstrut\\
		\hline
		200,300,300 & 0.609  & 0.658  & 0.253  & 0.281  \bigstrut[t]\\
		500,300,300 & 0.231  & 0.191  & 0.098  & 0.085  \\
		200,600,600 & 0.837  & 0.749  & 0.328  & 0.233  \\
		500,600,600 & 0.311  & 0.318  & 0.102  & 0.118  \bigstrut[b]\\
		\hline
	\end{tabular}%
	\label{tab:scca-init}%
\end{table}%

\begin{table}[htbp]\tiny
 	\centering
 		\caption{Comparison of A-ManPG and CoLaR for multiple sparse CCA \eqref{scca-mat-L21}. Covariance matrix: identity matrix}
 	\begin{tabular}{|c|c|c|c|c|c||c|c|c|c|c||c|c|c|c|c|}
 		\hline
 		\hline
 		& \multicolumn{5}{c||}{A-ManPG-1}       & \multicolumn{5}{c||}{CoLaR-1}     & \multicolumn{5}{c|}{CoLaR-100} \bigstrut\\\hline
        \multicolumn{16}{c|}{$(n,p,q)=(200,300,300)$} \\\hline
 		$b$     & 0.8   & 1     & 1.2   & 1.4   & 1.6   & 0.8   & 1     & 1.2   & 1.4   & 1.6   & 0.8   & 1     & 1.2   & 1.4   & 1.6 \bigstrut\\
 		\hline\hline
 		CPU   & 0.387  & 0.320  & 0.288  & {0.272 } & 0.274  & 0.763  & 0.760  & 0.725  & 0.667  & 0.619  & 5.071  & 5.020  & 4.844  & 4.724  & 4.663  \bigstrut\\
 	    lossu & 0.064  & 0.043  & 0.038  & {0.036 } & 0.045  & 0.094  & 0.062  & 0.049  & 0.046  & 0.051  & 0.081  & 0.051  & 0.038  & 0.041  & 0.047  \bigstrut\\
 		lossv & 0.075  & 0.053  & 0.044  & {0.047 } & 0.058  & 0.110  & 0.079  & 0.072  & 0.080  & 0.094  & 0.096  & 0.063  & 0.061  & 0.059  & 0.072  \bigstrut\\
 		nA & 44    & 23.5  & 15.5  & {10} & 10    & 64    & 32.5  & 18    & 12    & 10    & 64    & 30    & 16.5  & 10    & 10 \bigstrut\\
 		nB & 45.5  & 24.5  & 16    & {10} & 10    & 64    & 32.5  & 18    & 12    & 11    & 63    & 32    & 19    & 12    & 10 \bigstrut\\
 		$\rho_1$ & 0.919  & 0.907  & 0.900  & {0.897 } & 0.894  & 0.925  & 0.910  & 0.900  & 0.895  & 0.893  & 0.925  & 0.911  & 0.900  & 0.897  & 0.895  \bigstrut\\
 		$\rho_2$ & 0.863  & 0.833  & 0.822  & {0.813 } & 0.811  & 0.877  & 0.840  & 0.822  & 0.813  & 0.804  & 0.879  & 0.841  & 0.821  & 0.815  & 0.810  \bigstrut\\
 		\hline\hline
 \multicolumn{16}{c|}{$(n,p,q)=(500,300,300)$} \\\hline
 		CPU   & 0.300  & 0.275  & 0.264  & {0.263 } & 0.265  & 0.693  & 0.680  & 0.612  & 0.524  & 0.431  & 2.995  & 2.924  & 2.780  & 2.653  & 2.647  \bigstrut\\
 		lossu & 0.031  & 0.021  & 0.017  & {0.018 } & 0.020  & 0.032  & 0.018  & 0.017  & 0.018  & 0.019  & 0.032  & 0.018  & 0.015  & 0.017  & 0.020  \bigstrut\\
 		lossv & 0.031  & 0.019  & 0.018  & {0.019 } & 0.021  & 0.038  & 0.023  & 0.022  & 0.023  & 0.022  & 0.037  & 0.022  & 0.019  & 0.019  & 0.021  \bigstrut\\
 		nA & 62    & 25    & 14    & {10} & 10    & 67    & 28.5  & 16    & 10.5  & 10    & 63    & 28.5  & 15    & 10    & 10 \bigstrut\\
 		nB & 58    & 30    & 16.5  & {11} & 10    & 64.5  & 31.5  & 17.5  & 11    & 10    & 69    & 31.5  & 18    & 12    & 10 \bigstrut\\
 		$\rho_1$ & 0.907  & 0.901  & 0.898  & {0.897 } & 0.897  & 0.906  & 0.901  & 0.898  & 0.897  & 0.896  & 0.907  & 0.901  & 0.898  & 0.897  & 0.896  \bigstrut\\
 		$\rho_2$ & 0.833  & 0.816  & 0.808  & {0.804 } & 0.803  & 0.838  & 0.817  & 0.808  & 0.804  & 0.803  & 0.838  & 0.818  & 0.808  & 0.804  & 0.802  \bigstrut\\
 		\hline\hline
 \multicolumn{16}{c|}{$(n,p,q)=(200,600,600)$} \\\hline
 		CPU   & 1.329  & 1.133  & 1.062  & 1.021  & {0.992 } & 1.441  & 1.373  & 1.321  & 1.229  & 1.217  & 63.371  & 63.236  & 63.092  & 62.835  & 62.706  \bigstrut\\
 		lossu & 0.101  & 0.068  & 0.056  & 0.062  & {0.070 } & 0.182  & 0.117  & 0.094  & 0.097  & 0.103  & 0.141  & 0.095  & 0.071  & 0.071  & 0.085  \bigstrut\\
 		lossv & 0.091  & 0.069  & 0.059  & 0.057  & {0.073 } & 0.162  & 0.115  & 0.093  & 0.093  & 0.091  & 0.127  & 0.085  & 0.065  & 0.066  & 0.081  \bigstrut\\
 		nA & 54.5  & 31    & 18    & 12    & {10} & 91.5  & 49.5  & 23    & 16    & 12    & 78    & 37    & 18    & 12    & 10 \bigstrut\\
 		nB & 55    & 29.5  & 18    & 13    & {10.5} & 100   & 49.5  & 25.5  & 15    & 12    & 78    & 35    & 21    & 13.5  & 10 \bigstrut\\
 		$\rho_1$ & 0.926  & 0.915  & 0.910  & 0.906  & {0.903 } & 0.934  & 0.918  & 0.907  & 0.904  & 0.903  & 0.932  & 0.912  & 0.905  & 0.903  & 0.902  \bigstrut\\
 		$\rho_2$ & 0.879  & 0.843  & 0.821  & 0.804  & {0.798 } & 0.903  & 0.858  & 0.823  & 0.808  & 0.795  & 0.904  & 0.852  & 0.823  & 0.805  & 0.799  \bigstrut\\
 		\hline\hline
 \multicolumn{16}{c|}{$(n,p,q)=(500,600,600)$} \\\hline
 		CPU   & 1.094  & 1.019  & 0.989  & {0.976 } & 0.978  & 1.385  & 1.327  & 1.270  & 1.149  & 1.042  & 17.822  & 17.734  & 17.649  & 17.488  & 17.289  \bigstrut\\
 		lossu & 0.032  & 0.020  & 0.016  & {0.018 } & 0.020  & 0.041  & 0.023  & 0.018  & 0.017  & 0.019  & 0.041  & 0.024  & 0.017  & 0.018  & 0.020  \bigstrut\\
 		lossv & 0.032  & 0.018  & 0.014  & {0.015 } & 0.016  & 0.041  & 0.023  & 0.016  & 0.017  & 0.017  & 0.039  & 0.019  & 0.016  & 0.015  & 0.017  \bigstrut\\
 		nA & 78    & 36    & 16    & {12} & 10    & 98    & 37    & 17.5  & 12    & 10    & 99    & 40    & 17    & 12    & 10 \bigstrut\\
 		nB & 74.5  & 32    & 16    & {10} & 10    & 93    & 37.5  & 15.5  & 10    & 10    & 98.5  & 39.5  & 16    & 10    & 10 \bigstrut\\
 		$\rho_1$ & 0.914  & 0.906  & 0.904  & {0.903 } & 0.903  & 0.916  & 0.906  & 0.903  & 0.903  & 0.902  & 0.917  & 0.907  & 0.904  & 0.903  & 0.902  \bigstrut\\
 		$\rho_2$ & 0.846  & 0.822  & 0.807  & {0.803 } & 0.802  & 0.852  & 0.824  & 0.809  & 0.804  & 0.803  & 0.855  & 0.823  & 0.808  & 0.803  & 0.802  \bigstrut\\
 		\hline\hline
 	\end{tabular}%
 	\label{tab:scca-mat-1}%
 \end{table}%

\begin{table}[htbp]\tiny
	\centering
\caption{Comparison of A-ManPG and CoLaR for multiple sparse CCA \eqref{scca-mat-L21}. Covariance matrix: Topelitz matrix}
	\begin{tabular}{|c|c|c|c|c|c||c|c|c|c|c||c|c|c|c|c|}
		\hline\hline
		& \multicolumn{5}{c||}{A-ManPG-1}       & \multicolumn{5}{c||}{CoLaR-1}     & \multicolumn{5}{c|}{CoLaR-100} \bigstrut\\	\hline
		\multicolumn{16}{c|}{$(n,p,q)=(200,300,300)$} \\\hline
        $b$     & 0.8   & 1     & 1.2   & 1.4   & 1.6   & 0.8   & 1     & 1.2   & 1.4   & 1.6   & 0.8   & 1     & 1.2   & 1.4   & 1.6 \bigstrut\\
		\hline\hline
		CPU   & 0.380  & 0.327  & 0.292  & 0.283  & {0.282 } & 0.791  & 0.761  & 0.729  & 0.666  & 0.622  & 8.229  & 8.079  & 7.931  & 7.801  & 7.725  \bigstrut\\
		lossu & 0.069  & 0.045  & 0.043  & 0.049  & {0.064 } & 0.103  & 0.079  & 0.069  & 0.070  & 0.085  & 0.088  & 0.054  & 0.043  & 0.049  & 0.061  \bigstrut\\
		lossv & 0.075  & 0.057  & 0.046  & 0.050  & {0.060 } & 0.116  & 0.079  & 0.065  & 0.067  & 0.075  & 0.096  & 0.066  & 0.056  & 0.055  & 0.062  \bigstrut\\
		nA & 43    & 26    & 15.5  & 12    & {10} & 61.5  & 31    & 18.5  & 12    & 10    & 62.5  & 31.5  & 17    & 12    & 10 \bigstrut\\
		nB & 44.5  & 27    & 16    & 12    & {10} & 64.5  & 36    & 20    & 14    & 12    & 57    & 30    & 19    & 12    & 10 \bigstrut\\
		$\rho_1$ & 0.921  & 0.911  & 0.906  & 0.902  & {0.898 } & 0.925  & 0.912  & 0.905  & 0.902  & 0.900  & 0.926  & 0.912  & 0.906  & 0.902  & 0.899  \bigstrut\\
		$\rho_2$ & 0.864  & 0.835  & 0.818  & 0.803  & {0.794 } & 0.869  & 0.839  & 0.818  & 0.803  & 0.797  & 0.875  & 0.838  & 0.814  & 0.800  & 0.792  \bigstrut\\
		\hline\hline
        \multicolumn{16}{c|}{$(n,p,q)=(500,300,300)$} \\\hline
		CPU   & 0.310  & 0.287  & 0.266  & {0.261 } & 0.260  & 0.707  & 0.667  & 0.646  & 0.492  & 0.431  & 3.220  & 3.160  & 3.067  & 2.858  & 2.839  \bigstrut\\
		lossu & 0.025  & 0.015  & 0.010  & {0.010 } & 0.010  & 0.029  & 0.017  & 0.012  & 0.013  & 0.014  & 0.030  & 0.017  & 0.012  & 0.010  & 0.012  \bigstrut\\
		lossv & 0.027  & 0.016  & 0.012  & {0.010 } & 0.012  & 0.031  & 0.015  & 0.012  & 0.011  & 0.013  & 0.035  & 0.019  & 0.013  & 0.012  & 0.014  \bigstrut\\
		nA & 54    & 27    & 14.5  & {10} & 10    & 60.5  & 25.5  & 15    & 12    & 10    & 63.5  & 28    & 16    & 10    & 10 \bigstrut\\
		nB & 56.5  & 28    & 16    & {10} & 10    & 63    & 31    & 18    & 10    & 10    & 65.5  & 33.5  & 17.5  & 11    & 10 \bigstrut\\
		$\rho_1$ & 0.905  & 0.899  & 0.896  & {0.896 } & 0.895  & 0.906  & 0.900  & 0.897  & 0.896  & 0.896  & 0.906  & 0.900  & 0.897  & 0.896  & 0.895  \bigstrut\\
		$\rho_2$ & 0.835  & 0.819  & 0.810  & {0.807 } & 0.806  & 0.838  & 0.820  & 0.810  & 0.807  & 0.805  & 0.840  & 0.822  & 0.810  & 0.807  & 0.806  \bigstrut\\
		\hline\hline
    \multicolumn{16}{c|}{$(n,p,q)=(200,600,600)$} \\\hline
		CPU   & 1.427  & 1.214  & 1.120  & 1.048  & {1.034 } & 1.504  & 1.445  & 1.343  & 1.273  & 1.272  & 64.845  & 64.700  & 64.465  & 64.460  & 64.255  \bigstrut\\
		lossu & 0.077  & 0.055  & 0.050  & 0.051  & {0.059 } & 0.158  & 0.108  & 0.077  & 0.079  & 0.090  & 0.106  & 0.068  & 0.056  & 0.059  & 0.069  \bigstrut\\
		lossv & 0.079  & 0.063  & 0.044  & 0.044  & {0.047 } & 0.158  & 0.112  & 0.112  & 0.105  & 0.116  & 0.105  & 0.075  & 0.055  & 0.052  & 0.064  \bigstrut\\
		nA & 60    & 35    & 20    & 12    & {10} & 114   & 56    & 31.5  & 16    & 12    & 92.5  & 45    & 20    & 12    & 10 \bigstrut\\
		nB & 59.5  & 33    & 20    & 12    & {10} & 104   & 53.5  & 25    & 16    & 12    & 86    & 39    & 20    & 12.5  & 10 \bigstrut\\
		$\rho_1$ & 0.925  & 0.911  & 0.903  & 0.900  & {0.897 } & 0.936  & 0.915  & 0.901  & 0.897  & 0.894  & 0.933  & 0.913  & 0.902  & 0.899  & 0.896  \bigstrut\\
		$\rho_2$ & 0.879  & 0.842  & 0.815  & 0.797  & {0.789 } & 0.896  & 0.853  & 0.823  & 0.796  & 0.788  & 0.900  & 0.851  & 0.816  & 0.796  & 0.786  \bigstrut\\
		\hline\hline
    \multicolumn{16}{c}{$(n,p,q)=(500,600,600)$} \\\hline
		CPU   & 1.142  & 1.070  & 1.029  & {1.019 } & 1.011  & 1.419  & 1.349  & 1.290  & 1.178  & 1.071  & 16.082  & 15.965  & 15.844  & 15.795  & 15.676  \bigstrut\\
		lossu & 0.033  & 0.022  & 0.018  & {0.018 } & 0.020  & 0.041  & 0.022  & 0.015  & 0.014  & 0.015  & 0.042  & 0.022  & 0.019  & 0.019  & 0.022  \bigstrut\\
		lossv & 0.031  & 0.018  & 0.014  & {0.014 } & 0.017  & 0.034  & 0.018  & 0.012  & 0.011  & 0.012  & 0.040  & 0.019  & 0.013  & 0.013  & 0.016  \bigstrut\\
		nA & 79.5  & 37.5  & 16    & {11} & 10    & 92.5  & 38.5  & 18.5  & 12    & 10    & 93.5  & 38    & 16    & 12    & 10 \bigstrut\\
		nB & 77.5  & 34.5  & 16    & {12} & 10    & 91    & 36    & 17.5  & 11    & 10    & 93.5  & 38    & 17    & 12    & 10 \bigstrut\\
		$\rho_1$ & 0.913  & 0.904  & 0.902  & {0.900 } & 0.898  & 0.913  & 0.904  & 0.902  & 0.900  & 0.899  & 0.915  & 0.904  & 0.902  & 0.900  & 0.899  \bigstrut\\
		$\rho_2$ & 0.840  & 0.816  & 0.806  & {0.801 } & 0.799  & 0.846  & 0.818  & 0.806  & 0.802  & 0.800  & 0.847  & 0.819  & 0.807  & 0.800  & 0.798  \bigstrut\\
		\hline\hline
	\end{tabular}%
 	\label{tab:scca-mat-2}%
\end{table}%

 \begin{table}[htbp]\tiny
 	\centering
 	\caption{Comparison of A-ManPG and CoLaR for multiple sparse CCA \eqref{scca-mat-L21}. Covariance matrix: sparse inverse matrix}
 	\begin{tabular}{|c|c|c|c|c|c||c|c|c|c|c||c|c|c|c|c|}
 		\hline\hline
 		& \multicolumn{5}{c||}{A-ManPG-1}       & \multicolumn{5}{c||}{CoLaR-1}     & \multicolumn{5}{c|}{CoLaR-100} \bigstrut\\\hline
 		$b$     & 0.8   & 1     & 1.2   & 1.4   & 1.6   & 0.8   & 1     & 1.2   & 1.4   & 1.6   & 0.8   & 1     & 1.2   & 1.4   & 1.6 \bigstrut\\
 		\hline\hline
        \multicolumn{16}{c}{$(n,p,q)=(200,300,300)$} \\\hline
 		CPU   & 0.810  & 0.654  & 0.576  & {0.538 } & 0.547  & 0.960  & 0.947  & 0.909  & 0.791  & 0.790  & 10.378  & 10.265  & 10.106  & 10.110  & 10.074  \bigstrut\\
 		lossu & 0.088  & 0.080  & 0.091  & {0.113 } & 0.138  & 0.178  & 0.151  & 0.135  & 0.130  & 0.147  & 0.130  & 0.107  & 0.099  & 0.114  & 0.148  \bigstrut\\
 		lossv & 0.115  & 0.111  & 0.127  & {0.157 } & 0.196  & 0.200  & 0.180  & 0.179  & 0.171  & 0.192  & 0.140  & 0.125  & 0.128  & 0.137  & 0.166  \bigstrut\\
 		nA & 43    & 24.5  & 16    & {12} & 11    & 79.5  & 50.5  & 34    & 23.5  & 18    & 59    & 36    & 23    & 17    & 13 \bigstrut\\
 		nB & 39    & 24.5  & 16    & {12} & 10    & 71.5  & 47    & 30    & 22    & 15    & 53    & 32    & 17    & 14    & 12 \bigstrut\\
 		$\rho_1$ & 0.919  & 0.909  & 0.899  & {0.893 } & 0.887  & 0.928  & 0.915  & 0.902  & 0.893  & 0.884  & 0.924  & 0.911  & 0.902  & 0.895  & 0.889  \bigstrut\\
 		$\rho_2$ & 0.854  & 0.829  & 0.813  & {0.803 } & 0.795  & 0.883  & 0.857  & 0.838  & 0.824  & 0.810  & 0.867  & 0.841  & 0.819  & 0.804  & 0.793  \bigstrut\\
 		\hline\hline
        \multicolumn{16}{c}{$(n,p,q)=(500,300,300)$} \\\hline
 		CPU   & 0.574  & 0.513  & 0.494  & {0.472 } & 0.453  & 0.940  & 0.906  & 0.833  & 0.746  & 0.721  & 4.681  & 4.619  & 4.564  & 4.424  & 4.404  \bigstrut\\
 		lossu & 0.038  & 0.036  & 0.040  & {0.051 } & 0.065  & 0.046  & 0.044  & 0.046  & 0.054  & 0.068  & 0.042  & 0.043  & 0.048  & 0.061  & 0.076  \bigstrut\\
 		lossv & 0.035  & 0.029  & 0.032  & {0.042 } & 0.052  & 0.050  & 0.039  & 0.036  & 0.045  & 0.049  & 0.040  & 0.029  & 0.033  & 0.039  & 0.045  \bigstrut\\
 		nA & 46    & 25.5  & 14.5  & {10} & 10    & 64.5  & 38    & 23.5  & 14.5  & 11    & 57    & 30    & 15    & 11    & 10 \bigstrut\\
 		nB & 47.5  & 26    & 16    & {12} & 10    & 76.5  & 42    & 25.5  & 18    & 13    & 66    & 38    & 19    & 13    & 11 \bigstrut\\
 		$\rho_1$ & 0.907  & 0.902  & 0.899  & {0.897 } & 0.896  & 0.909  & 0.903  & 0.900  & 0.897  & 0.894  & 0.907  & 0.902  & 0.898  & 0.895  & 0.893  \bigstrut\\
 		$\rho_2$ & 0.824  & 0.812  & 0.803  & {0.800 } & 0.797  & 0.833  & 0.818  & 0.810  & 0.805  & 0.799  & 0.829  & 0.815  & 0.807  & 0.801  & 0.795  \bigstrut\\
 		\hline\hline
        \multicolumn{16}{c}{$(n,p,q)=(200,600,600)$} \\\hline
 		CPU   & 2.129  & 1.906  & 1.789  & {1.683 } & 1.613  & 1.793  & 1.671  & 1.614  & 1.529  & 1.485  & 65.508  & 65.392  & 65.229  & 65.143  & 65.040  \bigstrut\\
 		lossu & 0.168  & 0.160  & 0.164  & {0.164 } & 0.178  & 0.323  & 0.271  & 0.257  & 0.252  & 0.268  & 0.211  & 0.184  & 0.183  & 0.219  & 0.273  \bigstrut\\
 		lossv & 0.142  & 0.127  & 0.119  & {0.128 } & 0.156  & 0.349  & 0.297  & 0.283  & 0.288  & 0.317  & 0.175  & 0.148  & 0.135  & 0.150  & 0.168  \bigstrut\\
 		nA & 50    & 29.5  & 20    & {13} & 12    & 131   & 81.5  & 55    & 31    & 23    & 90    & 44.5  & 22.5  & 16    & 13 \bigstrut\\
 		nB & 50    & 30    & 19    & {14} & 10    & 136.5 & 89    & 61    & 41    & 26    & 88.5  & 50    & 26.5  & 19.5  & 12.5 \bigstrut\\
 		$\rho_1$ & 0.922  & 0.909  & 0.902  & {0.899 } & 0.896  & 0.941  & 0.926  & 0.916  & 0.905  & 0.897  & 0.931  & 0.913  & 0.903  & 0.898  & 0.894  \bigstrut\\
 		$\rho_2$ & 0.870  & 0.840  & 0.818  & {0.801 } & 0.788  & 0.914  & 0.881  & 0.847  & 0.820  & 0.800  & 0.888  & 0.854  & 0.828  & 0.811  & 0.796  \bigstrut\\
 		\hline\hline
        \multicolumn{16}{c}{$(n,p,q)=(500,600,600)$} \\\hline
 		CPU   & 1.777  & 1.605  & 1.536  & {1.467 } & 1.454  & 1.690  & 1.635  & 1.598  & 1.546  & 1.486  & 39.566  & 39.440  & 39.320  & 39.247  & 39.180  \bigstrut\\
 		lossu & 0.044  & 0.035  & 0.039  & {0.049 } & 0.061  & 0.075  & 0.057  & 0.054  & 0.057  & 0.063  & 0.052  & 0.043  & 0.046  & 0.052  & 0.064  \bigstrut\\
 		lossv & 0.045  & 0.033  & 0.034  & {0.042 } & 0.053  & 0.058  & 0.047  & 0.049  & 0.051  & 0.059  & 0.051  & 0.037  & 0.037  & 0.043  & 0.053  \bigstrut\\
 		nA & 69    & 33    & 17    & {12} & 10    & 120.5 & 63.5  & 34.5  & 20    & 13    & 83.5  & 40    & 18    & 13    & 10 \bigstrut\\
 		nB & 64    & 33.5  & 19    & {12} & 10    & 112.5 & 57.5  & 29    & 18    & 14    & 92    & 39.5  & 20.5  & 12.5  & 10 \bigstrut\\
 		$\rho_1$ & 0.909  & 0.901  & 0.897  & {0.896 } & 0.895  & 0.919  & 0.908  & 0.901  & 0.898  & 0.895  & 0.914  & 0.903  & 0.898  & 0.896  & 0.895  \bigstrut\\
 		$\rho_2$ & 0.833  & 0.813  & 0.802  & {0.796 } & 0.792  & 0.849  & 0.822  & 0.807  & 0.800  & 0.793  & 0.838  & 0.816  & 0.803  & 0.794  & 0.789  \bigstrut\\
 		\hline\hline
 	\end{tabular}%
 	\label{tab:scca-mat-3}%
 \end{table}%